\RequirePackage{etex}
\documentclass[lettersize,journal]{IEEEtran}
\usepackage{caption} 
\usepackage{amsmath}
\usepackage{etex}
\usepackage{array}
\usepackage{enumitem}
\usepackage{textcomp}
\usepackage{verbatim}
\usepackage{todonotes}
\usepackage{pifont}

\usepackage{shortex}
\allowdisplaybreaks

\begin{document}

\title{Deeply Learned Robust Matrix Completion for Large-scale Low-rank Data Recovery}

\author{HanQin Cai, Chandra Kundu, Jialin Liu, and Wotao Yin
\thanks{Some preliminary results were previously presented in NeurIPS \cite{cai2021learned}.}
\thanks{
H.Q.~Cai and J.~Liu are with the School of Data, Mathematical, and Statistical Sciences and the Department of Computer Science, University of Central Florida, Orlando, FL 32816, USA. (e-mail: hqcai@ucf.edu; jialin.liu@ucf.edu)}
\thanks{C.~Kundu is with the School of Data, Mathematical, and Statistical Sciences, University of Central Florida, Orlando, FL 32816, USA. (e-mail: chandra.kundu@ucf.edu)}
\thanks{
W.~Yin is with Damo Academy, Alibaba US, Bellevue, WA 98004, USA. (e-mail: wotao.yin@alibaba-inc.com)
}
\thanks{
This work was partially supported by NSF DMS 2304489. The authors contributed equally. Corresponding author: HanQin Cai.
}
}


\maketitle

\begin{abstract}
Robust matrix completion (RMC) is a widely used machine learning tool that simultaneously tackles two critical issues in low-rank data analysis: missing data entries and extreme outliers. This paper proposes a novel \textit{scalable} and \textit{learnable} non-convex approach, coined Learned Robust Matrix Completion (LRMC), for large-scale RMC problems. LRMC enjoys low computational complexity with linear convergence. Motivated by the proposed theorem, the free parameters of LRMC can be effectively learned via deep unfolding to achieve optimum performance. Furthermore, this paper proposes a \textit{flexible} feedforward-recurrent-mixed neural network framework that extends deep unfolding from fixed-number iterations to infinite iterations. The superior empirical performance of LRMC is verified with extensive experiments against state-of-the-art on synthetic datasets and real applications, including video background subtraction, ultrasound imaging, face modeling, and cloud removal from satellite imagery. 
\end{abstract}


\section{Introduction} \label{sec:introduction}
Dimension reduction is one of the foundation tools for modern data science and machine learning. 
In real applications, one major challenge for large-scale dimension reduction is \textit{missing data entries}, which often occurs due to environmental, hardware, or time constraints. Given the data in the form of low-rank matrices, the complex task of inferring these missing data entries from the available data is known as matrix completion \cite{candes2009exact, candes2010power}, whose significance extends across diverse machine learning tasks, including collaborative filtering \cite{bennett2007kdd,mongia2021matrix,cai2023ccs}, image processing \cite{chen2004recovering,hu2012fast,chen2022color}, signal processing \cite{qu2015accelerated,cai2019fast,cai2023structured,cai2025hsnld}, sensor localization \cite{singer2010uniqueness,tasissa2018exact,smith2025riemannian,kundu2025redg} and traffic time series \cite{chen2024Laplacian,chen2024forecasting}. 
The complication of matrix completion is further amplified when some of the partially observed data are corrupted by \textit{impulse noise} or \textit{extreme outliers}. To address these challenges, Robust Matrix Completion (RMC) emerges as a powerful solution, tackling low-rank data recovery problems with both missing entries and outlier corruptions simultaneously \cite{chen2013low,cherapanamjeri2017nearly,zeng2017outlier,cai2021robust,wang2024robust,cai2024rccs}. 

Suppose we partially observe some entries of a rank-$r$ matrix $\Xs \in \reals^{n_1 \times n_2}$, and some of the observations are grossly corrupted by sparse outliers, namely $\Ss$. Mathematically, the task of RMC aims to simultaneously recover the rank-$r$ matrix $\Xs$ and sparse matrix $\Ss$  from the observations
\[
\Pi_\Omega\Y := \Po{\Xs + \Ss},
\] 
where $\Po{\cdot}$ is a sampling operator that sets  $[\Po{\Y}]_{i,j} = [\Y]_{i,j}$ if $(i,j) \in \Omega$ and $[\Po{\Y}]_{i,j} = 0$ otherwise. In this study, the set of observed locations $\Omega$ is sampled independently according to the Bernoulli model with a probability $p \in (0,1]$.


RMC problems have been extensively studied in recent years, and vast classic-style algorithms have been proposed. While many of them solve the RMC problems effectively, they are often computationally too expensive, especially when the problem scales get larger and larger in this era of big data. Inspired by the recent success of deep unfolding in sparse coding \cite{gregor2010learning}, one can naturally extend such techniques to RMC algorithms. 
In particular, a classic iterative RMC algorithm can be parameterized and unfolded as a feedforward
neural network (FNN). Then, certain parameters of the algorithms, which are now also the parameters of the FNN, can be learned through backpropagation. 
Given a specific application of RMC, the problem instances often share similar key
properties, e.g., observation rate, rank, incoherence, and outlier density. Thus, the algorithm with learned parameters can have superior performance than its baseline. 
However, the existing learning-based RMC methods have two common issues. Firstly, they often invoke an expensive step of singular value decomposition (SVD) iteratively, in both the training and inference stages. Note that SVT costs as much as $\cO(n^3)$ flops. Hence, their scalability is questionable for large-scale RMC problems. Secondly, the existing unfolding framework learns the parameters for only a finite number of iterations. In case a user desires a better accuracy than the current learned algorithm can reach, the algorithm has to be relearned with more unfolded iterations. 
To overcome these issues, under some theoretical guidelines, this paper proposes a \textit{scalable}, \textit{learnable}, and \textit{flexible} non-convex framework for solving large-scale RMC problems in a highly efficient fashion.



\vspace{-0.1in}
\subsection{Related work and main contributions}
Robust matrix completion problems are known as robust principal component analysis (RPCA) with partial observations in the earlier works \cite{lin2010augmented,candes2011robust,chen2013low}. With convex relaxed models of RMC, the recovery can be guaranteed with at least $\cO(\mu r n \log^6(n))$ observations and no more than $\cO(1/\mu r)$ portion of outliers, where the incoherence parameter $\mu$ will later be formally defined in \Cref{sec:theoretical results}, along with assumptions of the observation and outlier patterns. However, the convex approaches suffer from theoretical sublinear convergence, and the typical per-iteration computational complexity is as high as $\cO(n^3)$ flops.
Many efficient non-convex approaches have been proposed for RMC problems later \cite{yi2016fast,zeng2017outlier, cherapanamjeri2017nearly,zhang2018robust,tong2021accelerating}. Under general settings, the non-convex algorithms typically provide linear convergence with $p\geq\cO(\mathrm{poly}(\mu r) \mathrm{polylog}(n)/n)$ random observations and $\alpha\leq\cO({1}/{\mathrm{poly}(\mu r)})$ portion of outliers. The computational complexities are as low as $\cO(p n^2 r+n r^2 + \alpha p n^2 \log(pn))$ flops \cite{yi2016fast,tong2021accelerating}. Note that the iteration complexity of GD \cite{yi2016fast} depends on the condition number $\kappa$ of $\Xs$, i.e., $\cO(\kappa\log(1/\varepsilon))$ iterations to find an $\varepsilon$-solution, which is improved to $\cO(\log(1/\varepsilon))$ iterations by ScaledGD \cite{tong2021accelerating}.

Deep unfolding has emerged as a powerful technique for accelerating iterative algorithms in various low-rank problems. Originating from the parameterization of the Iterative Shrinkage-Thresholding Algorithm (ISTA) for LASSO in LISTA \cite{gregor2010learning}, this approach has been extended to numerous problems and network architectures \cite{xin2016maximal, yang2016deep, metzler2017learned, zhang2018ista, adler2018learned, liu2019alista, monga2021algorithm, han2024wtdun, shah2024optimization, zou2024proximal, karan2024unrolled}. Another related technique, ``learning to learn,'' explores recurrent neural networks for parameterizing iterative algorithms, demonstrating promising results \cite{Andrychowicz2016, Wichrowska2017learned, Metz2019understanding, harrison2022closer, liu2023ms4l2o, he2024mathematics}.

While deep unfolding has been successfully applied to RPCA \cite{cohen2019deep, solomon2019deep, luong2021a, markowitz2022multimodal, joukovsky2023interpretable}, its application to the more challenging RMC problem remains limited. Existing RPCA methods often rely on computationally expensive singular value thresholding (SVT) \cite{cai2010singular} for low-rank approximation [12, 13], hindering scalability. Moreover, they typically focus on specific problem domains, such as ultrasound imaging or video processing, limiting their generalizability. Denise \cite{herrera2020denise} proposes a deep unfolding approach for positive semidefinite low-rank matrices, achieving significant speedups but with limited applicability to general RMC problems.

To address these limitations, this paper proposes a novel learning-based method, coined Learned Robust Matrix Completion (LRMC) and summarized in \Cref{algo:LRMC}, for large-scale RMC problems. Our main contributions are fourfold:
\begin{enumerate}[leftmargin=0.2in]
\item LRMC is \textit{scalable} and \textit{learnable}. It costs merely $\cO(pn^2r)$ flops and avoids the costly iterative SVD and partial sorting.
All operators of LRMC are differentiable with respect to the parameters, thus learnable through backpropagation.


\item By establishing an exact recovery guarantee for a special case of LRMC in Theorem~\ref{thm:main_theorem}, 
we theoretically reveal the superior potential of LRMC compared to its baseline. Specifically, it confirms that there exists a set of parameters for LRMC to outperform the baseline.


\item We introduce a novel feedforward-recurrent-mixed neural network (FRMNN) model for deep unfolding. FRMNN first unfolds a finite number of significant iterations into an FNN and learns the parameters layer-wise, then it learns an RNN for parameter updating rules for the
subsequent iterations. As later shown in Fig.~\ref{fig:nn-structure}, 
FRMNN extends the deep unfolding models to a \textit{flexible} number of iterations without retraining or losing performance.

\item The empirical advantages of LRMC are confirmed with extensive numerical experiments in Section~\ref{sec:numerical}. Notably, LRMC shows superior performance on large-scale real applications, e.g., video background subtraction and cloud removal in satellite imagery, that are forbiddingly expensive for existing learning-based RMC approaches.
\end{enumerate}

\vspace{-0.1in}
\subsection{Notation}
For any matrix $\BM \in \mathbb{R}^{n_1 \times n_2}$, the $(i,j)$-th entry is denoted by $[\BM]_{i,j}$. For any index set $\Omega \subseteq [n_1] \times [n_2]$, we define $\Omega_{(i,:)} \coloneqq \{ (i,j) \in \Omega \mid j \in [n_2] \}$ and $\Omega_{(:,j)} \coloneqq \{ (i,j) \in \Omega \mid i \in [n_1] \}$. We define the projection operator $\Po{\BM}$ such that $[\Po{\BM}]_{i,j} = [\BM]_{i,j}$ if $(i,j) \in \Omega$ and $[\Po{\BM}]_{i,j} = 0$ otherwise. The $i$-th singular value of a matrix is denoted by $\sigma_i(\BM)$, entrywise $\ell_1$-norm is denoted by $\|\BM\|_1 := \sum_{i,j} |[\BM]_{i,j}|$, spectral norm is denoted by $\|\BM\|_{2}=\sigma_1(\BM)$, Frobenius norm is denoted by $\|\BM\|_\fro := ( \sum_{i,j} [\BM]_{i,j}^2 )^{1/2}$, and nuclear norm is denoted by $\|\BM\|_* := \sum_i \sigma_i(\BM)$. The largest magnitude entry is denoted by $\|\BM\|_\infty := \max_{i,j} |[\BM]_{i,j}|$, the largest row-wise $\ell_2$-norm is denoted by $\|\BM\|_{2,\infty} := \max_i ( \sum_j [\BM]_{i,j}^2 )^{1/2}$, and the largest row-wise $\ell_1$-norm is denoted by $\|\BM\|_{1,\infty}:= \max_i \sum_j |[\BM]_{i,j}|$. Let $(\cdot)^\top$ denote the transpose operator, and $\lor$ denote logical disjunction, which takes the max of two operands. The condition number of $\Xs$ is denoted by $\kappa := \frac{\sigma_1(\Xs)}{\sigma_r(\Xs)}$.

\section{Proposed Approach}
Consider the following RMC objective function: 
\begin{equation} \label{eq:obj_func}
    \begin{split}
        &\minimize_{\L,\R,\S}~~f(\BL,\BR;\BS)=\frac{1}{2p}\|\Pi_{\Omega}(\L\R^\top+\S-\Y)\|_\fro^2 \cr
&\subject  \supp (\S) \subseteq \supp(\Ss) ,
    \end{split}
\end{equation}
where $p:=\frac{|\Omega|}{n_1n_2}$ is the sampling rate, $\L \in \reals ^ {n_1 \times r}$ and $\R \in \reals ^ {n_2 \times r}$ are two tall matrices, and $\S \in \reals ^ {n_1 \times n_2}$ is sparse. 
This formulation rewrites $\X = \L \R \T$ to avoid the non-convex low-rank constraint on $\X$. Note that $\BS=\Po{\BS}$ since the outliers are only present in the observed entries. 
To solve problem \eqref{eq:obj_func}, we propose a novel scalable and learnable algorithm, coined Learned Robust Matrix Completion (LRMC).  
It operates in two phases: initialization and iterative updates. The initialization phase leverages a modified spectral initialization, specifically tailored to the RMC problem with learnable parameters. In the phase of iterative updates, we alternatively update the sparse and low-rank components in the fashion of scaled gradient descent, inspired by ScaledGD \cite{tong2021accelerating}; however, we redesign the operators so that all parameters are differentiable and thus become learnable. 
LRMC is summarized as Algorithm~\ref{algo:LRMC} and is detailed in the following subsections.
For ease of presentation, the discussion starts with the iterative updates.

\begin{algorithm}
\caption{Learned Robust Matrix Completion (LRMC)} \label{algo:LRMC}
\begin{algorithmic}[1]
    \State \textbf{Input:} $\Pi_\Omega\Y := \Po{\Xs + \Ss}$:  observed data matrix; $p:=|\Omega|/n_1n_2$:  sampling rate; $r$: the rank of underlying low-rank matrix; $\{\zeta_k\}$: a set of \textit{learnable} thresholding parameters; $\{\eta_k\}$: a set of \textit{learnable} step sizes. 
    \State \algcomment{\# Initialization:}
    \State $\S_0 = \fcS_{\zeta_0}(\Pi_\Omega\Y)$
    \State $[\U_0, \Sigmab_0, \V_0] =  \SVD_r(\pinv\Po{\Y-\S_0})$
    \State $\L_0 = \U_0 \Sigmab_0^{1 / 2}$ and $\R_0 = \V_0 \Sigmab_0^{1 / 2}$
    \\
    \algcomment{\# Iterative updates:}
    \While{ Not\texttt{(Stopping Condition)} }
        \State $\S_{k+1} = \fcS_{\zeta_{k+1}}\left(\Po{\Y-\L_k \R_k\T}\right)$
        \State $\L_{k+1} = \L_k - \eta_{k+1}  \grad_{\L} f_k \cdot (\R_k^\top \R_k)^{-1}$     
        \State $\R_{k+1} = \R_k - \eta_{k+1} \grad_{\R} f_k \cdot (\L_k^\top \L_k)^{-1}$   
    \EndWhile
    \\
    \textbf{Output:} $\X=\L_K \R_K^{\top}$: the recovered low-rank matrix.
\end{algorithmic}
\end{algorithm}

\vspace{-0.1in}
\subsection{Iterative updates}
\noindent \textbf{Updating sparse component.} 
In the baseline algorithms, ScaledGD \cite{tong2021accelerating} and GD \cite{yi2016fast}, the sparse outlier matrix is updated by $\BS_{k+1} = \cT_{\widetilde{\alpha}} (\Pi_\Omega(\BY - \BL_k\BR_k^\top))$, where the sparsification operator $\cT_{\widetilde{\alpha}}$ retains only the $\widetilde{\alpha}$-fraction largest-in-magnitude  entries in each row and column: 
    \[
    &  [\mathcal{T}_{\widetilde{\alpha}}(\BM)]_{i,j}
    = \begin{cases} 
    [\BM]_{i,j}, & \text{if } |[\BM]_{i,j}| \geq |[\BM]_{i,:}^{(\widetilde{\alpha}n_2)}|  \\
     & ~~\text{ and   }  |[\BM]_{i,j}| \geq |[\BM]_{:,j}^{(\widetilde{\alpha}n_1)}|; \\
    0, & \text{otherwise}.
    \end{cases} 
    \]
Although it ensures the sparsity in  $\BS$ updating, $\cT_{\widetilde{\alpha}}$ requires an accurate estimate of $\alpha$ and its execution involves expensive partial sorting on each row and column. Note that partial sorting becomes even more costly when $\alpha$ is large or $\widetilde{\alpha}$ is much overestimated. Moreover, deep unfolding and parameter learning cannot be applied to $\cT_{\widetilde{\alpha}}$ since it is not differentiable with respect to its parameter. 
Hence, we aim to find an efficient and effectively learnable operator to replace $\cT_{\widetilde{\alpha}}$.

The well-known soft-thresholding $\cS_\zeta$: 
\begin{align}  \label{eq:soft-thresholding}
    [\cS_\zeta(\BM)]_{i,j} = \sign([\BM]_{i,j}) \cdot \max\{0,|[\BM]_{i,j}| - \zeta\}
\end{align}
catches our attention. Soft-thresholding has been applied as a proximal operator of vectorized $\ell_1$-norm in some RMC literature \cite{lin2010augmented}. However, a fixed threshold, which is determined by the regularization parameter, leads to only sublinear convergence of the algorithm. Inspired by AltProj, LISTA, and their followup works \cite{netrapalli2014non,accaltproj,gregor2010learning,xin2016maximal,yang2016deep,metzler2017learned,zhang2018ista,adler2018learned,liu2019alista,cai2020rapid,monga2021algorithm,hamm2022riemannian}, we seek for a set of thresholding parameters $\{ \zeta_k \}$ that let our algorithm outperform the baseline ScaledGD.

In fact, we find that the simple soft-thresholding:
\begin{align}  \label{eq:S updating}
    \BS_{k+1}=\cS_{\zeta_{k+1}}(\Pi_\Omega(\BY-\BL_k\BR_k^\top))
\end{align}
can provide a linear convergence if we carefully choose the thresholding parameters, and the selected $\{\zeta_k\}$ also ensure $\supp(\BS_k)\subseteq\supp(\BS_\star)$ at every iteration, i.e., no false-positive outliers. Moreover, the proposed algorithm with selected thresholds can converge even faster than ScaledGD under the exact same setting (see Theorem~\ref{thm:main_theorem} in the next section for a formal statement.) Nevertheless, the theoretical selection of $\{\zeta_k\}$ relies on some knowledge of $\BX_\star$, which is usually unknown to the user. Fortunately, these thresholds can be reliably learned using deep unfolding techniques.

\vspace{0.05in}
\noindent \textbf{Updating low-rank component.} 
By parameterizing the low-rank component as the product of a tall matrix and a fat matrix, the low rankness of $\BX$ is enforced automatically; thus the expensive SVD is not required in every iteration of the algorithm. 
Denote the loss function as:
\[
f_k := f(\L_k,\R_k; \S_k) = \frac{1}{2p}\| \Pi_\Omega(\L_k\R_k\T + \S_k - \Y) \|_\fro^2.
\] 
The gradients, with respect to $\BL$ and $\BR$, can be easily computed:
\begin{equation}
        \begin{split}   
    \grad_{\L} f_k &= \pinv \Pi_\Omega(\L_k\R_k\T + \S_k - \Y)\R_k, \\ 
    \grad_{\R} f_k &= \pinv \Pi_\Omega (\L_k\R_k\T + \S_k - \Y)\T\L_k. 
            \end{split}
\end{equation} 
We could apply a step of gradient descent on each of $\BL$ and $\BR$. However, \cite{tong2021accelerating,cai2025hsnld,giampouras2025opsa} finds the vanilla gradient descent approaches (e.g., \cite{yi2016fast,cai2023hsgd}) suffer from ill-conditioning and thus introduces the scaled terms $(\BR_k^\top\BR_k)^{-1}$ and $(\BL_k^\top \BL_k)^{-1}$ to overcome this weakness. In particular, we follow ScaledGD and update the low-rank component with:
\begin{equation} \label{eq:X updating}
\begin{split}
    \BL_{k+1}&=\BL_k-\eta_{k+1}\nabla_\BL f_k \cdot (\BR_k^\top\BR_k)^{-1}, \cr
    \BR_{k+1}&=\BR_k-\eta_{k+1}\nabla_\BR f_k \cdot (\BL_k^\top \BL_k)^{-1},
\end{split}
\end{equation}
where $\eta_{k+1}$ is the learnable step size at the $(k+1)$-th iteration. 
While ScaledGD and GD use a fixed step size through all iterations, \cite{takabe2020theoretical} suggests that a set of optimal step sizes $\{\eta_k\}$ can be learned for gradient descent methods. 
\cite[Theorem~III.4]{takabe2020theoretical} suggests that the optimal step sizes $\{\eta_k\}$ for finite-step gradient descents should be in the form of decayed cosine oscillation, which can be learned via deep unfolding.

\begin{remark}
LRMC assumes the prior knowledge of $\rank(\BX_\star)$, which is commonly used in many non-convex matrix recovery approaches. In some applications, the rank of $\BX_\star$ can be easily obtained from prior knowledge, e.g., Euclidean distance matrices are rank-$(d+2)$ for the system of $d$-dimensional points \cite{tasissa2018exact}. However, in other settings, the true rank may be difficult to determine, and LRMC is not designed to learn the true rank from data. A learnable framework that is robust to rank misestimation will be investigated in a follow-up project.
\end{remark}

\vspace{-0.1in}
\subsection{Initialization}
LRMC uses a modified spectral approximation 
with learnable soft-thresholding for the initialization. That is, we first initialize the sparse matrix by $\BS_0=\cS_{\zeta_0}(\Po{\BY})$, which should detect the obvious, large outliers. Next, for the low-rank component, we take $\BL_0=\BU_0\BSigma_0^{1/2}$ and $\BR_0=\BV_0\BSigma_0^{1/2}$, where $\BU_0\BSigma_0\BV_0^\top$ is the best rank-$r$ approximation (via truncated SVD, denoted by $\SVD_r$) of $\pinv\Po{\BY-\BS_0}$. Note that $\pinv$ is needed here to reweight the expectation of sampling operator \cite[Theorem~7]{recht2011simpler}. 
Clearly, the initial thresholding step is crucial for the quality of initialization. Thanks to soft-thresholding, its parameter $\zeta_0$ can be optimized through learning. 

For huge-scale problems where even a single truncated SVD is computationally prohibitive, one may replace the SVD step in initialization with some sketch-based low-rank approximations, e.g., CUR decomposition, whose computational costs are as low as $\cO(r^2n\log^2n)$ flops. While its stability lacks theoretical support when outliers appear, the empirical evidence shows that a single step of robust CUR decomposition can provide sufficiently good initialization \cite{cai2020rapid}.

\vspace{-0.1in}
\subsection{Computational complexity} \label{sec:complexity}

To compute \eqref{eq:S updating}, $\Po{\L_k\R_k^\top}$ involves $\cO(pn^2r)$ flops and $\Po{\Y-\L_k \R_k^\top}$ requires $pn^2$ flops, followed by another $pn^2$ flops for soft-thresholding. Therefore, the total complexity for this step is $\cO(pn^2r)$ flops.

To compute \eqref{eq:X updating}, \( \Po{\L_k \R_k^\top + \S_{k+1} - \Y} \) requires $2pn^2$ flops, \( \R_k^\top \R_k \) takes \( \mathcal{O}(n r^2) \) flops, inverting \( \R_k^\top \R_k \) requires \( \mathcal{O}(r^3) \) flops, the multiplication of \( \R_k (\R_k^\top \R_k)^{-1} \) involves \( \mathcal{O}(n r^2) \) flops, and thus calculating \( \Po{\L_k \R_k^\top + \S_{k+1} - \Y} \R_k (\R_k^\top \R_k)^{-1} \) requires \( \mathcal{O}(p n^2 r) \) flops. Finally, updating \( \L_{k+1} \) takes another \( 2nr \) flops for the step size and then subtraction. Similarly for updating \( \BR_{k+1} \). In total, computing \eqref{eq:X updating} takes 
$\cO(p n^2 r + nr^2)$ flops, provided \( r \ll n \).

To summarize, the per-iteration cost of LRMC is 
$
\cO(p n^2 r+n r^2)
$
flops which is slightly better than the state-of-the-art ScaledGD and GD under the same setting. Unlike ScaledGD and GD, the computational complexity of LRMC does not depend on the outlier sparsity parameter $\alpha$ which will be empirically verified in Fig.~\ref{fig:step_runtime} later. When the sampling rate $p$ is near the information-theoretic lower bound $\cO(r/n)$, the computational complexity becomes $\cO(n r^2)$ flops per iteration. 
In addition, LRMC enjoys a linear convergence which will be detailed in the next section.

\section{Theoretical results} \label{sec:theoretical results}
In this section, we present the recovery guarantee of LRMC. Since the baseline ScaledGD only provides the convergence guarantee for RPCA case, i.e., fully observed RMC, we use the same setting in the theorem for an apple-to-apple comparison. 
The results show that when the parameters are selected correctly, LRMC provably outperforms ScaledGD in the exact setting.

We start with some common assumptions for RMC/RPCA:
\setcounter{assumption}{0}
\setcounter{remark}{0}
\setcounter{theorem}{0}
\setcounter{lemma}{0}

\bassump
[Bernoulli sampling]
Each of the observed locations in $\Omega$ is drawn independently according to the Bernoulli model with probability $p\in(0,1]$. 
\eassump

\bassump
[$\mu$-incoherence of low-rank component] \label{as:incoherence}
$\BX_\star\in\mathbb{R}^{n_1 \times n_2}$ is a rank-$r$ matrix with $\mu$-incoherence, i.e.,
\[
    \|\BU_\star\|_{2,\infty}\leq \sqrt{\mu r/n_1} 
    \qquad\textnormal{and}\qquad
    \|\BV_\star\|_{2,\infty}\leq \sqrt{\mu r/n_2}
\]
for some constant $1\leq\mu\leq n$, where $\BU_\star\BSigma_\star\BV_\star^\top$ is the compact SVD of $\BX_\star$. 
\eassump

\bassump
[$\alpha$-sparsity of outlier component] \label{as:sparsity}
$\BS_\star\in\mathbb{R}^{n_1 \times n_2}$ is an $\alpha$-sparse matrix, i.e., there are at most $\alpha$ fraction of non-zero elements in each row and column of $\BS_\star$. In particular, we require $\alpha \lesssim \cO(\frac{1}{\mu r^{3/2}\kappa})$ for the guaranteed recovery, which matches the requirement for ScaledGD. We further assume the problem is well-posed, i.e., $\mu,r,\kappa \ll n$.
\eassump

\begin{remark}
Note that some applications may have a more specific structure for outliers, which may lead to a more suitable operator for $\BS$ updating in the particular applications. This work aims to solve the most common RMC model with a generic sparsity assumption. Nevertheless, our learning framework can adapt to another operator as long as it is differentiable. 
\end{remark}

By the rank-sparsity uncertainty principle, a matrix cannot be incoherent and sparse simultaneously \cite{chandrasekaran2011rank}. The above two assumptions ensure the uniqueness of the solution in RMC/RPCA. 
Now, we are ready to present the main theorem:

\begin{theorem}[Guaranteed recovery] \label{thm:main_theorem}
Suppose that $\BX_\star$ is a rank-$r$ matrix with $\mu$-incoherence and $\BS_\star$ is an $\alpha$-sparse matrix with $\alpha\leq\frac{1}{10^4\mu r^{3/2}\kappa}$. Let $p=1$, i.e., $\Omega=[n_1]\times[n_2]$. If we set the thresholding values $\zeta_0=\|\BX_\star\|_\infty$ and $\zeta_k=\|\BL_{k-1}\BR_{k-1}^\top-\BX_\star\|_\infty$ for $k\geq 1$, the iterates of LRMC satisfy
\begin{gather}
\|\BL_k\BR_k^\top- \BX_\star\|_\fro\leq 0.03 (1-0.6\eta)^k \sigma_r(\BX_\star), \\ 
\supp(\BS_k) \subseteq \supp(\BS_\star),
\end{gather}
with the step sizes $\eta_k=\eta\in[\frac{1}{4},\frac{8}{9}]$.
\end{theorem}
\begin{proof}
The proof of Theorem~\ref{thm:main_theorem} is deferred to \Cref{sec:proofs}.
\end{proof}

Essentially, Theorem~\ref{thm:main_theorem} states that there exists a set of selected thresholding values $\{\zeta_k\}$ that allows one to replace the sparsification operator $\cT_{\widetilde{\alpha}}$ in ScaledGD with the simpler soft-thresholding operator $\cS_\zeta$ and maintains the same linear convergence rate $1-0.6\eta$ under the exact same assumptions. Note that the theoretical choice of parameters relies on the knowledge of $\BX_\star$---an unknown factor. Thus, Theorem~\ref{thm:main_theorem} can be read as a proof of the existence of the appropriate parameters.

Moreover, Theorem~\ref{thm:main_theorem} shows two advantages of LRMC: 
\begin{enumerate}[leftmargin=0.2in]
    \item Under the very same assumptions and constants, we allow the step sizes $\eta_k$ to be as large as $\frac{8}{9}$; in contrast, ScaledGD has the step sizes no larger than $\frac{2}{3}$ \cite[Theorem~2]{tong2021accelerating}. That is, LRMC can provide faster convergence under the same sparsity condition by allowing larger step sizes. 
    \item With the selected thresholding values, $\cS_\zeta$ in LRMC is effectively a projection onto $\supp(\BS_\star)$, which matches our sparsity constraint in objective \eqref{eq:obj_func}. That is, $\cS_\zeta$ takes out the larger outliers, leaves the smaller outliers, and preserves all good entries---no false-positive outlier in $\BS_k$. In contrast, $\cT_{\widetilde{\alpha}}$ necessarily yields some false-positive outliers in the earlier stages of ScaledGD, which drag the algorithm when outliers are relatively small.  
\end{enumerate}

\vspace{0.05in}
\noindent\textbf{Technical innovation.} The main challenge to our analysis is to show both the distance error metric (i.e., $\dist(\BL_k,\BR_k;\BL_\star,\BR_\star)$, later defined in Section~\ref{sec:proofs}) and $\ell_\infty$ error metric (i.e., $\|\BL_k\BR_k^\top-\BX_\star\|_\infty$) are linearly decreasing. While the former takes some minor modifications from the proof of ScaledGD, the latter is rather challenging and utilizes several new technical lemmas. Note that ScaledGD shows $\|\BL_k\BR_k^\top-\BX_\star\|_\infty$ is always bounded 
but not necessarily decreasing, which is insufficient for LRMC. 

\section{Parameter learning}
Theorem~\ref{thm:main_theorem} shows the existence of ``good" parameters $\{\zeta_k\},\{\eta_k\}$ and, in this section, we describe how to obtain such parameters via machine learning techniques.

\vspace{0.05in}
\noindent\textbf{Feed-forward neural network.} Classic deep unfolding methods unroll an iterative algorithm and truncate it into a fixed number, say $K$, iterations. Applying such an idea to our model, we regard each iteration of Algorithm~\ref{algo:LRMC} as a layer of a neural network and regard the variables $\BL_{k}, \BR_{k}, \BS_{k} $ as the units of the $k$-th hidden layer. The top part of Fig.~\ref{fig:nn-structure} demonstrates the structure of such feed-forward neural network (FNN). The $k$-th layer is denoted as $\mathcal{L}_{k}$.  Based on (\ref{eq:S updating}) and (\ref{eq:X updating}), it takes $\BY$ as an input and has two parameters $\zeta_k, \eta_k$ when $k \geq 1$. The initial layer $\mathcal{L}_0$ is special, it takes $\BY$ and $r$ as inputs, and it has only one parameter $\zeta_0$. For simplicity, we use $\Theta = \{ \{\zeta_k\}_{k=0}^K, \{\eta_k\}_{k=1}^K \}$ to represent all parameters in this neural network.

\vspace{0.05in}
\noindent\textbf{Training.} Given a training data set $\mathcal{D}_{\mathrm{train}}$ consisting of ($\BY, \BX_\star$) pairs (observation, groundtruth), one can train the neural network and obtain parameters $\Theta$ by minimizing:
\[
    \minimize_{\Theta} \mathbb{E}_{(\BY, \BX_\star) \sim \mathcal{D}_{\mathrm{train}}}\|\BL_K(\BY,\Theta)\BR_K(\BY,\Theta)^\top - \BX_\star\|^2_\fro.
\]
Directly minimizing this loss function is called end-to-end training, and it can be easily implemented on deep learning platforms nowadays. In this paper, we adopt a more advanced training technique named \emph{layer-wise training} or \emph{curriculum learning}, which has been proven as a powerful tool for training deep unfolding models \cite{chen2018theoretical,chen2020training}. The process of layer-wise training is divided into $K+1$ stages:
\begin{itemize}[leftmargin=0.2in]
    \item Training $0$-th layer with $\minimize_{\Theta}\mathbb{E}\|\BL_0\BR_0^\top-\BX_\star\|_\fro^2$.
    \item Training $1$-st layer with $\minimize_{\Theta}\mathbb{E}\|\BL_1\BR_1^\top-\BX_\star\|_\fro^2$.
    \item $\cdots\cdots$
    \item Training final layer with  $\minimize_{\Theta}\mathbb{E}\|\BL_K\BR_K^\top-\BX_\star\|_\fro^2$.
\end{itemize}

\begin{figure}[t]
\centering
\resizebox{\columnwidth}{!}{
\input{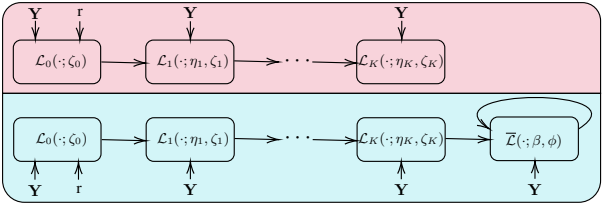}
}
\vspace{-0.1in}
\caption{A high-level structure comparison between classic FNN-based deep unfolding (top) and proposed FRMNN-based deep unfolding (bottom). In the diagrams, $\mathcal{L}_k$ denotes the $k$-th layer of FNN and $\overline{\mathcal{L}}$ is a layer of RNN.
}
\label{fig:nn-structure}
\end{figure}




\vspace{0.05in}
\noindent\textbf{Feedforward-recurrent-mixed neural network.} One disadvantage of the aforementioned FNN model is its fixed number of layers. If one wants to take further steps and obtain higher accuracy after training a neural network, one may have to retrain the neural network once more. Recurrent neural network (RNN) has tied parameters over different layers and, consequently, is extendable to infinite layers. However, we find that the starting iterations play significant roles in the convergence (validated in Section~\ref{sec:numerical} later) and their parameters should be trained individually. Thus, we propose a hybrid model that is demonstrated in the bottom part of Fig.~\ref{fig:nn-structure}, named as Feedforward-recurrent-mixed neural network (FRMNN). We use an RNN appended after a $K$-layer FNN. When $k \geq K$, in the $(k-K)$-th loop of the RNN layer, we follow the same calculation procedures with Algorithm~\ref{algo:LRMC} and determine the parameters $\zeta_k$ and $\eta_k$ by
\begin{equation}
    \eta_k = \beta \eta_{k-1} \quad\textnormal{ and }\quad \zeta_k = \phi \zeta_{k-1}.
\end{equation}
Thus, all RNN layers share the common parameters $\beta$ and $\phi$.

The training of FRMNN follows in two phases:
\begin{itemize}[leftmargin=0.2in]
    \item Training the $K$-layer FNN with layer-wise training.
    \item Fixing the FNN and searching RNN parameters $\beta$ and $\phi$ to minimize the convergence metric at the $(\overline{K}-K)$-th layer of RNN for some $\overline{K}>K$:
    \begin{equation}
    \hspace{-0.1in}
    \minimize_{\beta,\phi}  \mathbb{E}_{(\BY, \BX_\star) \sim \mathcal{D}_{\mathrm{train}}}\|\BL_{\overline{K}}(\beta,\phi)(\BR_{\overline{K}}(\beta,\phi))^\top - \BX_\star\|^2_\fro.
    \end{equation}
\end{itemize}
Our new model provides the flexibility of training a neural network with finite $\overline{K}$ layers and testing it with infinite layers. Consequently, the stop condition of LRMC has to be \texttt{(Run $K$ iterations)} if the parameters are trained via FNN model; and the stop condition can be \texttt{(Error $<$ Tolerance)} if our FRMNN model is used.

In this paper, we use stochastic gradient descent (SGD) in the layer-wise training stage and use grid search to obtain $\beta$ and $\phi$ since there are only two parameters. Specifically, we search over $\beta, \phi \in \{0.1,0.2, \cdots, 0.9 \}$ and select the pair that minimizes the reconstruction error. Moreover, we find that picking $K+3\leq\overline{K}\leq K+5$ is empirically good.

\vspace{0.05in}
\noindent\textbf{Training hyperparameters.}  Hyperparameters in our training procedure involve the initial thresholding parameters $\{\zeta_k\}_{k=0}^{K-1}$, initial step size parameters $\{\eta_k\}_{k=0}^{K-1}$, their learning rates, and the number of training epochs. We initialize the first threshold parameter from the scale of the training data as $\zeta_0 = \frac{\|\BX_\star\|_\infty}{n}$. Then set $\zeta_k = c \zeta_{k-1}$ with a fixed decay factor $c \in (0,1)$ for $k\geq 1$. The step size parameters are initialized as $\eta_k = 1/p$. We train all parameters using stochastic gradient descent (SGD) with adaptive learning rates. For the threshold parameters, the learning rate for the first layer is set as $\lr_{\zeta_0} = \gamma_0/\ell_\gamma$, and for deeper layers we use $\lr_{\zeta_k} = \zeta_{k-1}/\ell_\gamma$ for $k\geq 1$. For the synthetic dataset, we set the base learning rate $\gamma_0 = 0.1$ and the scaling constants $\ell_\gamma = 5000$. For the step size parameters, we use a fixed learning rate $\lr_{\eta_k} = 0.1$ for all $k$. Each layer $k$ is first pre-trained for 1000 epochs by updating only $(\zeta_k, \eta_k)$, followed by 1000 epochs of joint training in which all parameters up to layer $k$ are updated. To avoid numerical instability, if a threshold $\zeta_k$ becomes negative during training, we reset it to its initial positive value.

\section{Numerical experiments} \label{sec:numerical}
In this section, we verify the empirical performance of LRMC against state-of-the-art methods such as classic-style ScaledGD \cite{tong2021accelerating} and learning-based CORONA \cite{cohen2019deep}, with synthetic and real datasets.  For the classic-style algorithms, the parameters are carefully hand-tuned for their optimal results. 
All learning-based methods are trained on a workstation equipped with an Intel i9-9940X CPU, 128GB RAM, and two Nvidia A6000 GPUs. 
All runtime comparisons are conducted with Matlab on a workstation equipped with an Intel i9-12900H CPU and 32GB RAM. The code for LRMC is available
online at \url{https://github.com/chandrakundu/LRMC}.

\vspace{-0.1in}
\subsection{Synthetic datasets}
For both the learning and inference stages of LRMC, we generate the underlying low-rank matrix $\Xs = \Ls \Rs^\top$ as a product of two Gaussian random matrices $\Ls,\Rs \in \reals^{n \times r}$. The location of the observation $\Omega$ is sampled according to the Bernoulli model with a probability $p$. The support of the sparse matrix $\Ss$ are sampled uniformly without replacement from $\Omega$ and the magnitudes of these non-zero entries obey the uniform distribution bounded by $[-\mathbb{E}|[\Xs]_{i,j}|, \mathbb{E}|[\Xs]_{i,j}|]$. The input for RMC algorithms is then $\Pi_\Omega\Y = \Pi_\Omega\Xs + \Ss \in \reals^{n \times n}$. 
All inference runtime results presented for synthetic data are averaged over 20 randomly generated instances. 
For the two phases of algorithm learning: During the layer-wise training phase (i.e., FNN training), we use SGD with a batch size of 1. For the recurrent parameter search phase (i.e., RNN training), we employ a grid search with a step size of 0.1. A new training instance $(\BY, \Xs)$ is generated for each step of SGD, and 20 instances are generated for each step of the grid search. 

\begin{figure}[ht]
  \centering
    \includegraphics[width=0.38\textwidth]{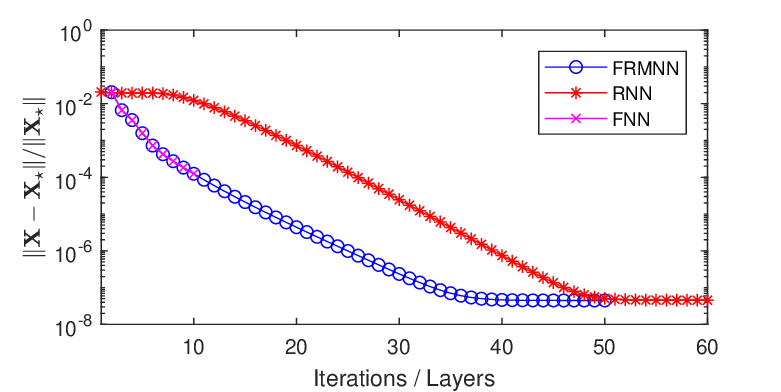}

    \vspace{-0.08in}
    \caption{Convergence comparison for FNN-based, RNN-based, and FRMNN-based learning.}
    \vspace{-1em}
    \label{fig:fpn}
\end{figure}

\vspace{0.05in}
\noindent\textbf{Unfolding models.} 
We compare the performance of three different unfolding models with LRMC: classic FNN, RNN, and the proposed feedforward-recurrent-mixed neural network (FRMNN). The FNN model unrolls $K=10$ iterations of LRMC and RNN model directly starts the training on the second phase of FRMNN, i.e., $K=0$. FRMNN starts with $K=10$ layers of FNN and is followed by an RNN trained with $\overline{K}=15$. Full observation problems, i.e., $p=1$, are used for this experiment. The test results are summarized in Fig.~\ref{fig:fpn}. 
One can see FRMNN model extends the classic FNN model to infinite layers without performance reduction, while the RNN model drops the convergence performance significantly. 
Note that the convergence curves of both FRMNN and RNN stall at the accuracy $\cO(10^{-8})$, which is the machine precision since single precision is used in this experiment.

\begin{figure}[ht]
  \centering
  \begin{minipage}{0.49\columnwidth}
    \centering
    \includegraphics[width=\columnwidth]{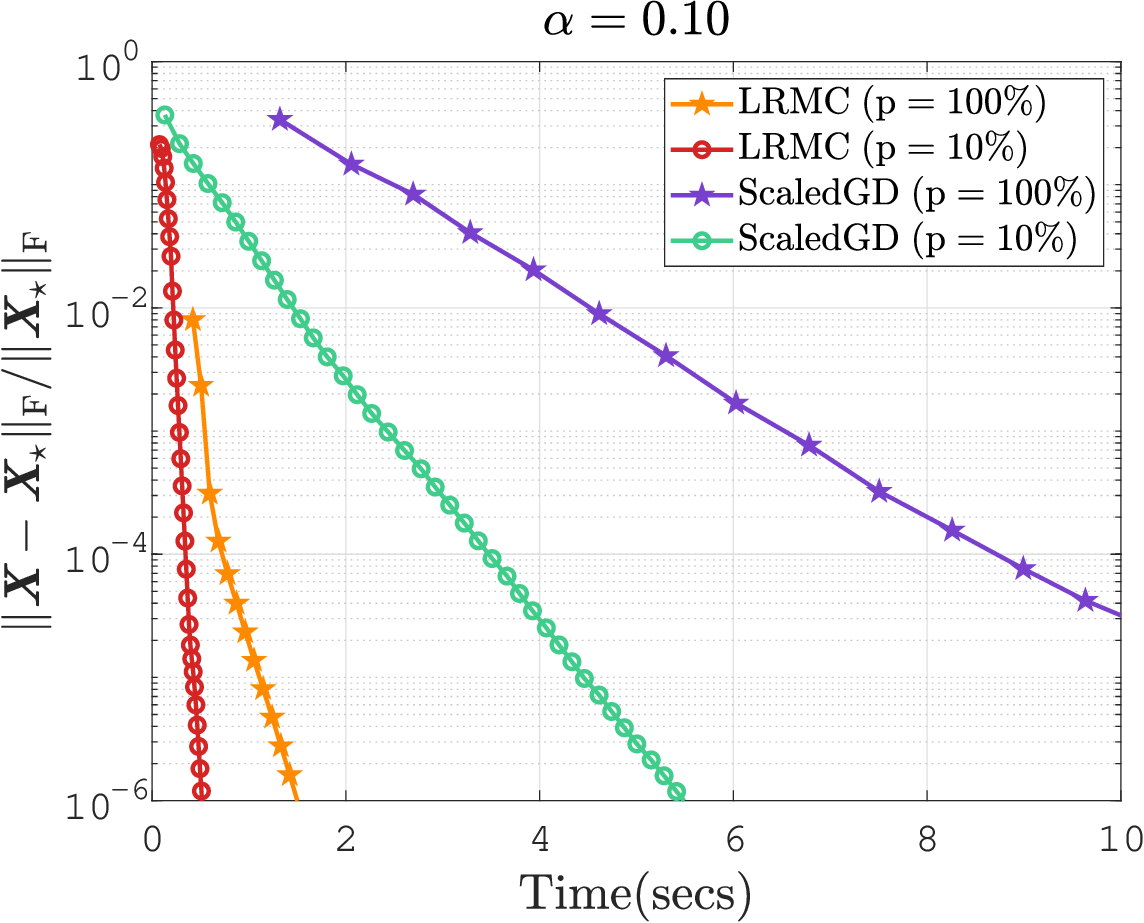}
    \label{fig:conv_alpha0.1}
  \end{minipage}%
  \hfill
  \begin{minipage}{0.49\columnwidth}
    \centering
    \includegraphics[width=\columnwidth]{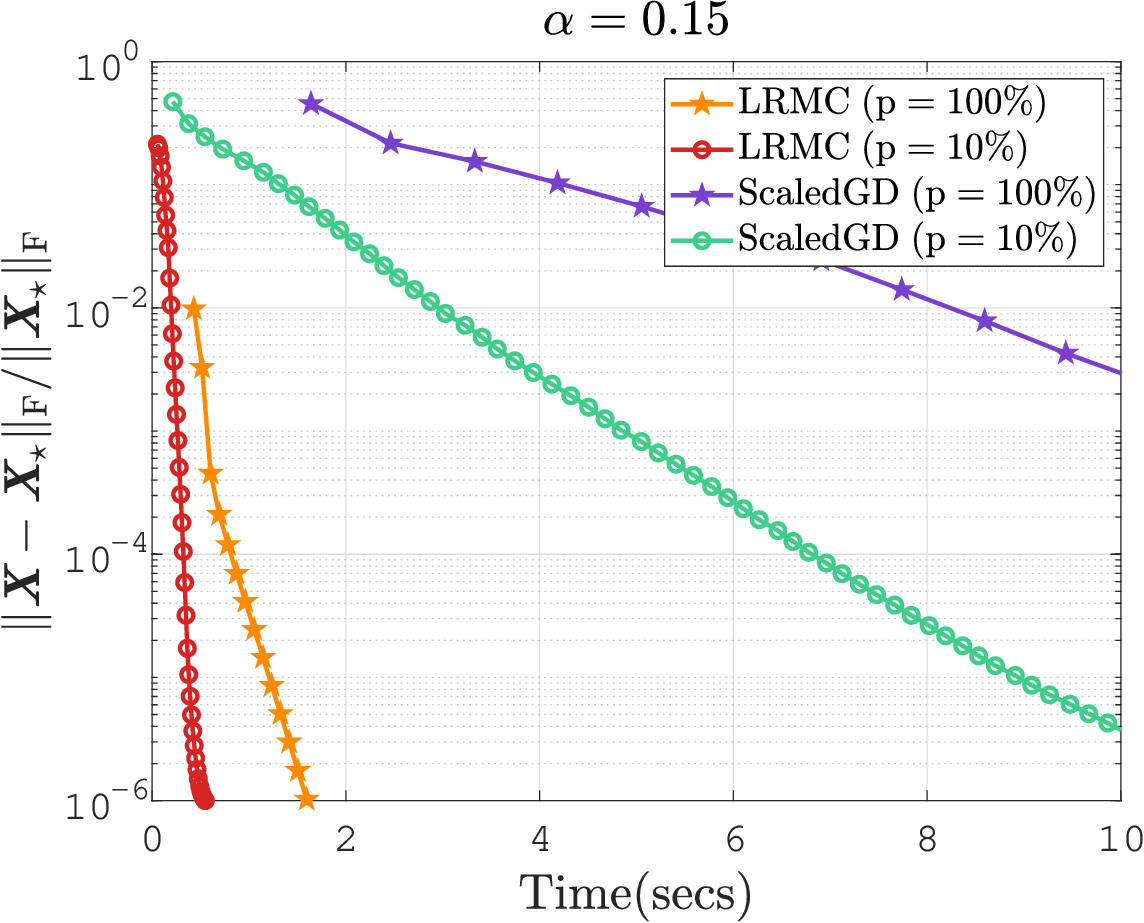}
    \label{fig:conv_alpha0.15}
  \end{minipage}%

  \begin{minipage}{0.49\columnwidth}
    \centering
    \includegraphics[width=\columnwidth]{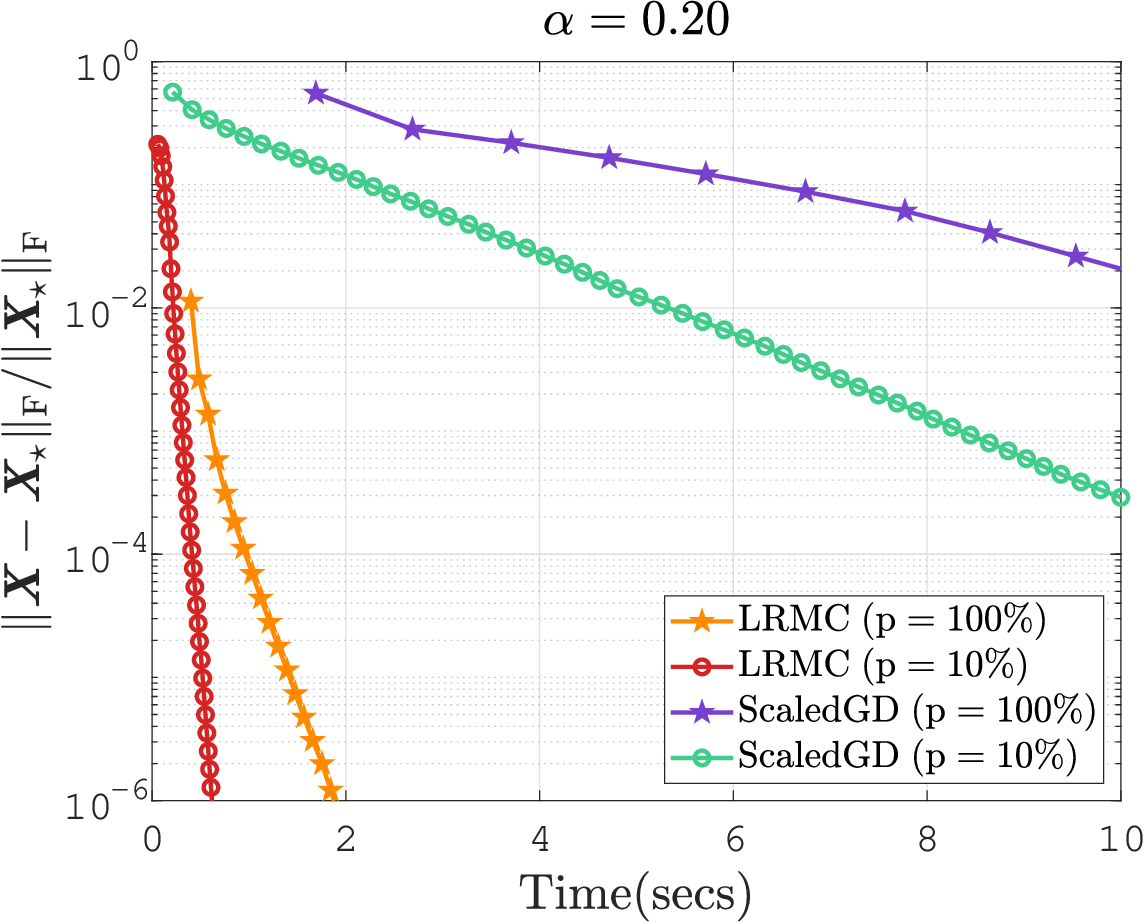}
    \label{fig:conv_alpha0.2}
  \end{minipage}%
  \hfill
  \begin{minipage}{0.49\columnwidth}
    \centering
    \includegraphics[width=\columnwidth]{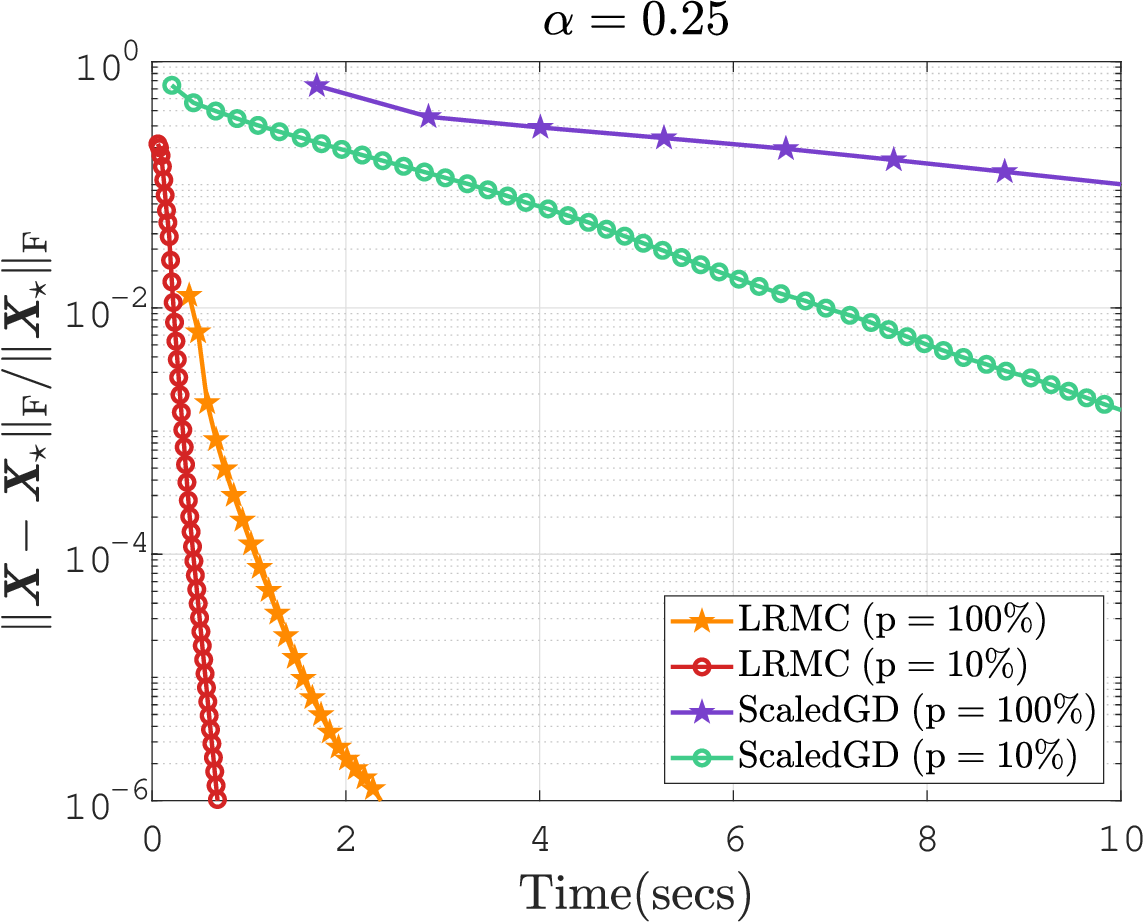}
    \label{fig:conv_alpha0.25}
  \end{minipage}%

\vspace{-0.1in}
  \caption{Convergence comparison for LRMC and ScaledGD with varying outlier sparsity $\alpha$. Problem dimension $n=3000$ and rank $r=5$.}
  \label{fig:conv-iter}
\end{figure}

\vspace{0.05in}
\noindent\textbf{Computational efficiency.} We present the speed advantage of LRMC at the inference stage, provided appropriate training. 
Fig.~\ref{fig:conv-iter} compares the convergence behavior of LRMC and the baseline ScaledGD. Both algorithms achieve linear convergence, which matches Theorem~\ref{thm:main_theorem}. LRMC consistently demonstrates faster convergence than ScaledGD accross the observation rates. When the outlier sparsity is higher, the advantage of LRMC is even more significant. This is further verified with Fig.~\ref{fig:step_runtime}. On the left, we see the per-iteration runtime of ScaledGD is slower when $\alpha$ becomes larger. In contrast, the per-iteration runtime of LRMC is insensitive to $\alpha$ and is significantly faster than ScaledGD. On the right, we find the total runtime of LRMC is substantially faster than ScaledGD. 
Note that when the observation rate $p$ is smaller, the per-iteration runtime is faster, which matches the complexity analysis in Section~\ref{sec:complexity}; however, it will require more iterations to achieve the same accuracy. Table~\ref{tab:varying p} thoroughly tests and reports the runtime and number of iterations to achieve the same accuracy with varying observation rate.



\begin{figure}[ht]
  \centering
  \begin{minipage}{0.49\columnwidth}
    \centering
    \includegraphics[width=\columnwidth]{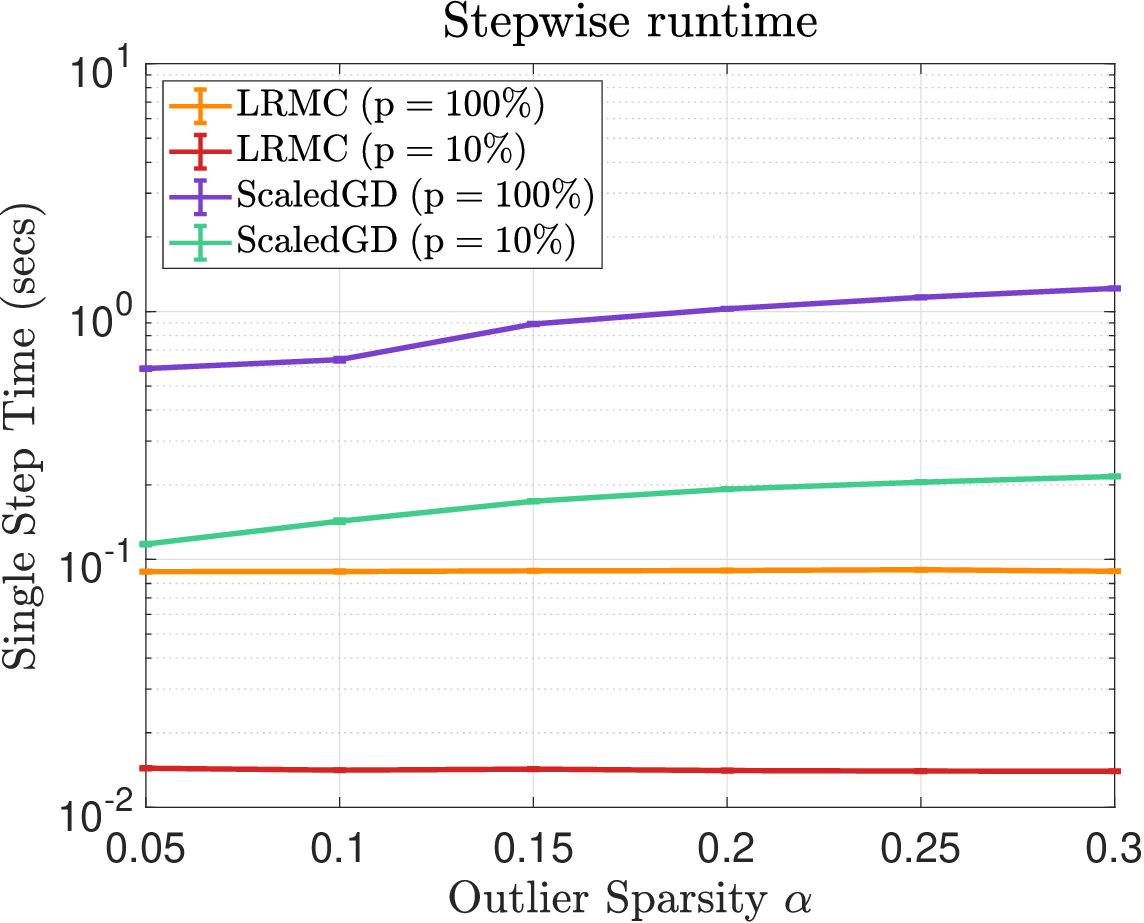}

  \end{minipage}%
  \hfill
  \begin{minipage}{0.48\columnwidth}
    \centering
    \includegraphics[width=\columnwidth]{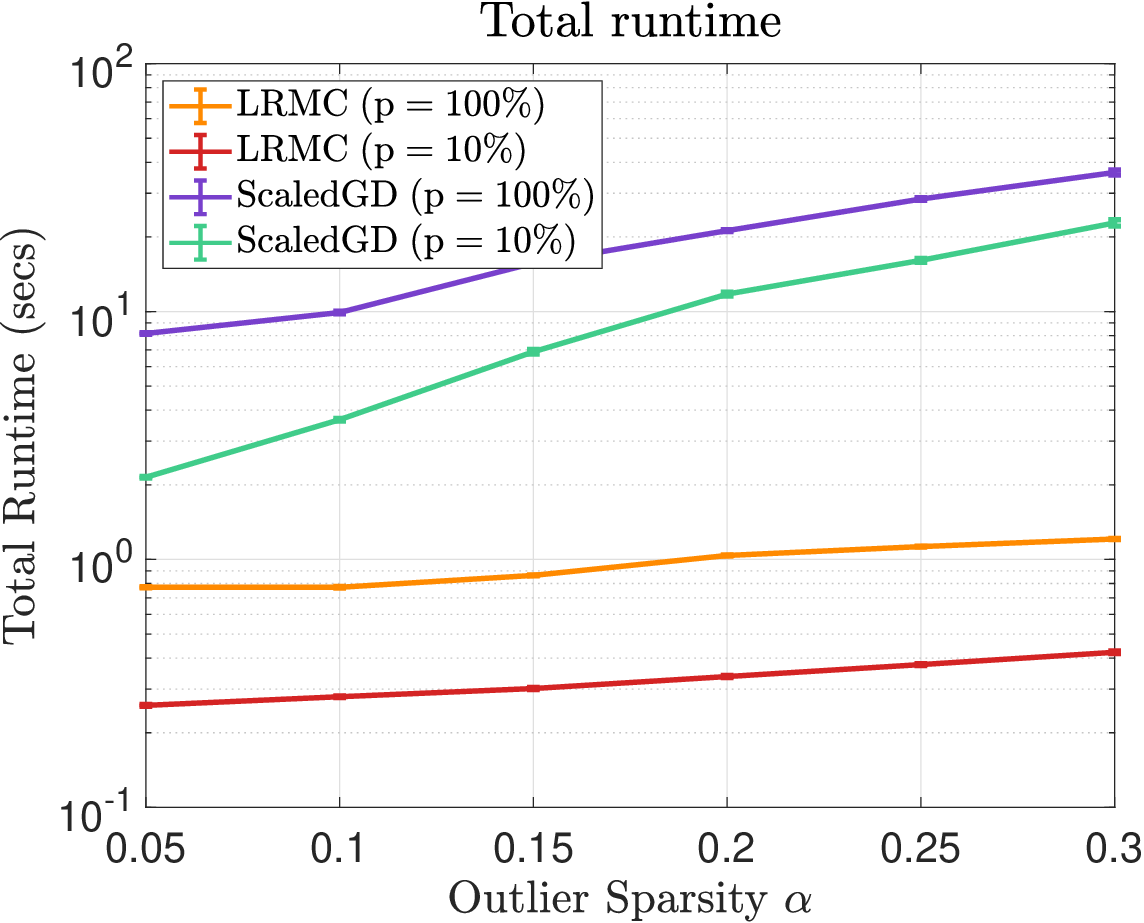}
   
  \end{minipage}%
  \caption{Runtime comparison for LRMC and ScaledGD for varying outlier sparsity $\alpha$. Problem dimension $n=3000$ and rank $r=5$. \textbf{Left:} Single step runtime averaged over 100 iterations, excluding initialization. \textbf{Right:} Total runtime. All algorithms halt when $\|\X - \Xs\|_\text{F} / \|\Xs\|_\text{F} < 10^{-6}$.}
  \label{fig:step_runtime}
\end{figure}

\begin{table}[ht]
\caption{Runtime and number of iterations comparisons for LRMC and ScaledGD with varying observation rate $p$. All algorithms halt when $\|\X - \Xs\|_\text{F} / \|\Xs\|_\text{F} < 10^{-6}$.} \label{tab:varying p} 
\centering
\begin{tabular}{c|c|c|c|c}
\toprule
$p$  & \multicolumn{2}{c|}{LRMC} & \multicolumn{2}{c}{ScaledGD} \\
(\%) & Time (secs) & Iterations & Time (secs) & Iterations \\
\midrule
10  & 0.47  & 28  & 5.37  & 36 \\
20  & 0.63  & 19  & 5.80  & 23 \\
30  & 0.68  & 15  & 7.12  & 21 \\
40  & 0.81  & 15  & 8.85  & 20 \\
50  & 0.90  & 14  & 11.00 & 20 \\
60  & 1.01  & 13  & 12.79 & 20 \\
70  & 1.23  & 13  & 14.53 & 20 \\
80  & 1.29  & 13  & 16.18 & 19 \\
90  & 1.36  & 12  & 17.49 & 19 \\
100 & 1.48  & 12  & 16.40 & 18 \\
\bottomrule
\end{tabular}

\end{table}

\vspace{0.05in}
\noindent\textbf{Recoverability.} 
To evaluate LRMC's robustness against outliers, we generated $20$ problem instances with varying outlier density levels and compared its recoverability to ScaledGD. Table~\ref{tab:robustness} shows that LRMC has superior recoverability to ScaledGD against high-density outliers, despite $10\%$ or $100\%$ observation cases. However, more observation offers more redundancy, thus better robustness for the same algorithm.

\begin{table}[ht]
\centering
\caption{Recoverability with varying outlier sparsity $\alpha$.} \label{tab:robustness}
\tabcolsep=0.18cm
\begin{tabular}{l|c|c|c|c|c|c}
\toprule
$\alpha$  & 0.35 & 0.4   & 0.45  & 0.5   & 0.55 & 0.6  \cr \midrule
LRMC ($p = 100\%$)& \textbf{20}/20   & \textbf{20}/20 & \textbf{20}/20 & \textbf{20}/20 & \textbf{19}/20 & \textbf{17}/20 \cr
LRMC ($p = 10\%$)& \textbf{20}/20  & \textbf{20}/20 & \textbf{20}/20 & 14/20 & 12/20 & 0/20 \cr
ScaledGD ($p = 100\%$)& \textbf{20}/20 & \textbf{20}/20 & \textbf{20}/20 & 0/20  & 0/20 & 0/20 \cr
ScaledGD ($p = 10\%$)& \textbf{20}/20 & 0/20 & 0/20 & 0/20  & 0/20 & 0/20 \cr 
\bottomrule
\end{tabular}
\end{table}

\vspace{0.05in}
\noindent\textbf{Robustness with additive noise.} 
We further evaluate LRMC when observations are corrupted by both sparse outliers and additive dense Gaussian noise. Fig.~\ref{fig:noise_robustness} reports the recovered SNR, defined as $\mathrm{SNR}=10 \log_{10}(\|\BX_\star\|_\fro^2 / \|\BX_\star - \BX\|_\fro^2)$,  against the input SNR of the noisy matrix before adding $10\%$ of outliers. The recovery quality scales linearly with the input SNR, even in the presence of these sparse outliers. Notably, training with noise yields a consistent improvement in recovery quality when the observation is full ($p=100\%$). However, under a lower observation rate ($p=10\%$), the benefit of noise augmentation during training becomes negligible, as the recovery difficulty is dominated by the missing entries rather than the additive Gaussian noise.

\begin{figure}[ht]
    \centering
    \includegraphics[width=0.7\columnwidth]{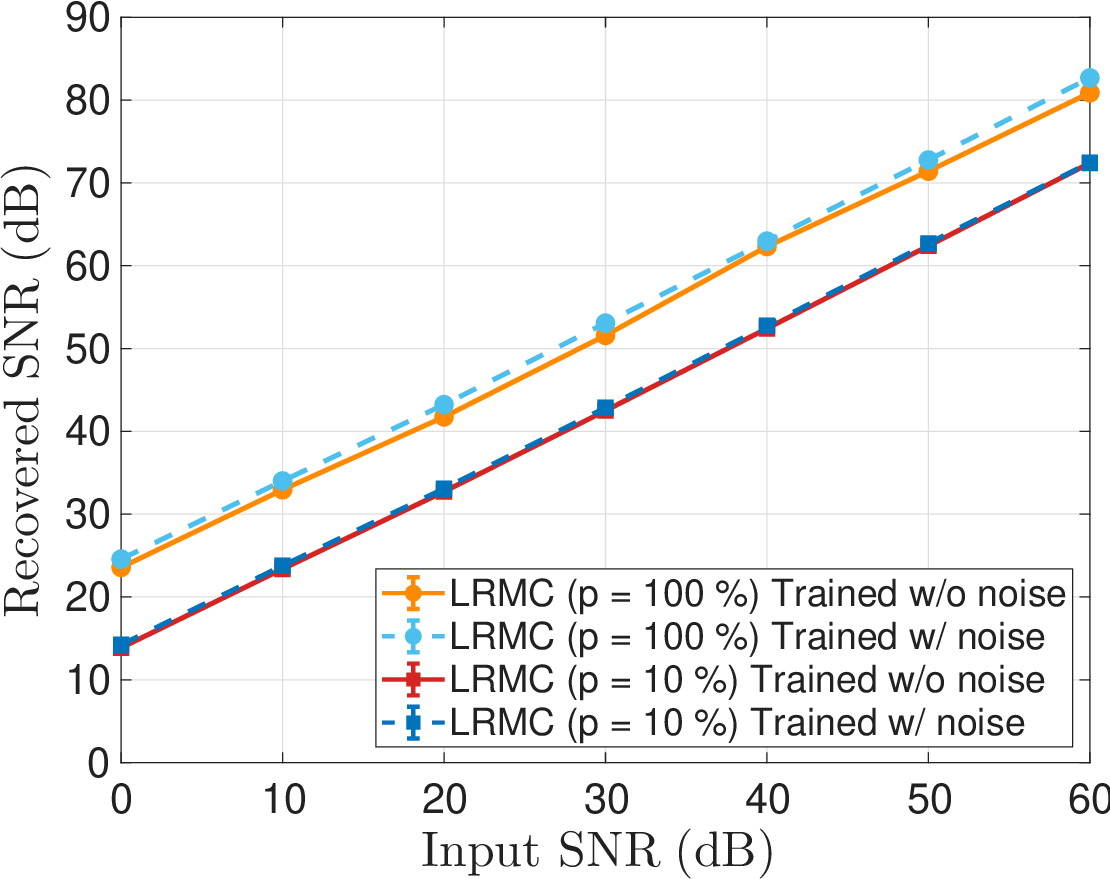}
    \caption{Comparison of recovered SNR against input SNR for additive Gaussian noise. In addition, $10\%$ outliers are added. 
    Models are trained with or without noise under full ($p=100\%$) or partial ($p=10\%$) observation.}
    \label{fig:noise_robustness}
\end{figure}

\vspace{-0.1in}
\subsection{Real datasets}
\noindent\textbf{Video background subtraction.} 
In this section, we implement LRMC for video background subtraction, employing the VIRAT video dataset \cite{oh2011large} as our benchmark dataset. VIRAT comprises a diverse collection of videos characterized by colorful scenes with static backgrounds. The videos have been separated into training and testing sets, with $305$ videos for training and $5$ videos for verification\footnote{
Available at \url{https://viratdata.org}. 
The tested videos correspond to the original video numbers in VIRAT: 1: S\_000203\_04\_000903\_001086. 2: S\_010202\_01\_000055\_000147. 3: S\_040000\_05\_000668\_000703. 4: S\_050000\_04\_000640\_000690. 5: S\_050201\_10\_001992\_002056.}. 
All videos have the same frame size $320\times180$ but different frame numbers, which are reported in \Cref{tab:video}. We first convert all videos
to grayscale. Next, each frame of a video is vectorized and becomes a matrix column, and all
frames of a video form a data matrix. The static backgrounds are the low-rank component of the data matrices and moving objects can be viewed as outliers, thus we can separate the backgrounds and foregrounds via RMC. Given partial observation, $\BS$ can not be fully recovered; thus the foreground is defined as $\BY-\BX_\mathrm{output}$ in this case. The runtime results are provided in \Cref{tab:video}, with selected visual results presented in Fig.~\ref{fig:visualization_virat}. The visual results of LRMC look crispy for both observation rates and LRMC with $p=10\%$ achieves over $10\times$ speedup than ScaledGD.

\begin{table}[ht]
\centering
\caption{Video details and runtime results for video background subtraction. All algorithms halt when $\|\X_{k}-\BX_{k-1}\|_\fro/\|\BX_{k-1}\|_\fro <10^{-2}$ or after 32 iterations. 
} 
\label{tab:video}

\tabcolsep=0.1cm
 \begin{tabular}{c|c|cc|cc} 
\toprule
 \textsc{Video} &\textsc{Frame} & \multicolumn{4}{c}{\textsc{Runtime} (secs)} \cr
 \textsc{Name} & \textsc{Count} &  \multicolumn{2}{c|}{LRMC} & \multicolumn{2}{c}{ScaledGD } \cr
 \textsc{} & \textsc{}  & $p = 100\%$ & $p = 10\%$  & $p = 100\%$ & $p = 10\%$ \cr
\midrule
 Video 1  & $5482$    &   $26.94$ &  $\textbf{10.39}$    & $193.65$   & $93.29$ \cr
 Video 2  & $2024$  &   $8.76$   &  $\textbf{3.45}$    & $66.63$ & $32.21$ \cr
 Video 3   & $1038$  &   $4.03$ &  $\textbf{1.68}$    & $31.30$  & $17.61$  \cr
 Video 4  & $1492$  &   $6.53$  &  $\textbf{1.99}$    & $57.48$ & $27.86$ \cr
 Video 5   & $1907$  &   $8.04$ &  $\textbf{2.10}$    & $68.87$  & $31.77$  \cr 
\bottomrule
\end{tabular}
\end{table}

\begin{figure}[ht]
    \centering
    \begin{minipage}[b]{0.19\columnwidth}
      \centering
      \phantom{blank} \vspace{1 mm}
      \includegraphics[width=\columnwidth]{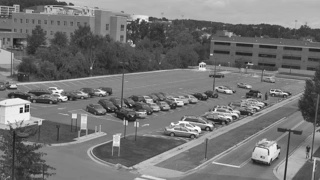} \vspace{1 mm}  
      \includegraphics[width=\columnwidth]{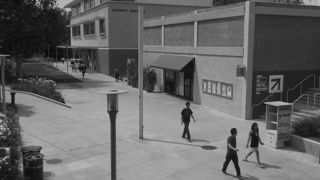} \vspace{1 mm}
        \includegraphics[width=\columnwidth]{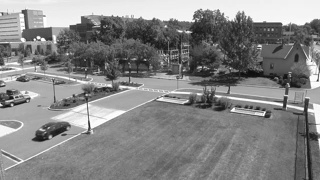} \vspace{1 mm}
        \includegraphics[width=\columnwidth]{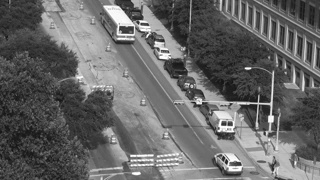} \vspace{1 mm}
        \includegraphics[width=\columnwidth]{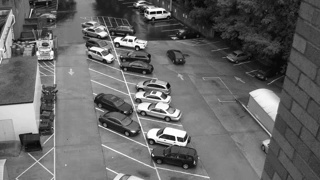} 
    \end{minipage}
    \hfill
    \begin{minipage}[b]{0.19\columnwidth}
        \centering
        \phantom{blank} \vspace{1 mm}
        \includegraphics[width=\columnwidth]{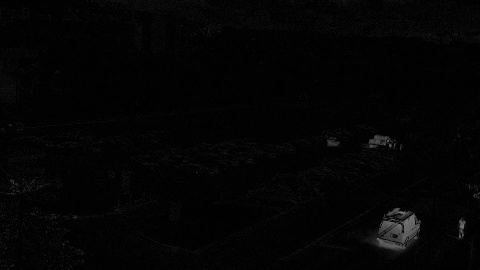} \vspace{1 mm}
        \includegraphics[width=\columnwidth]{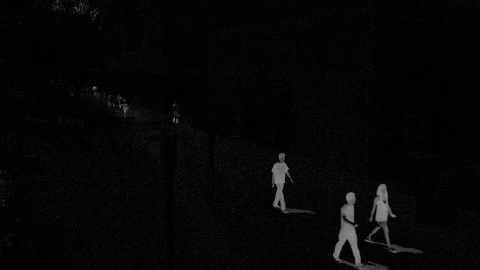} \vspace{1 mm}
        \includegraphics[width=\columnwidth]{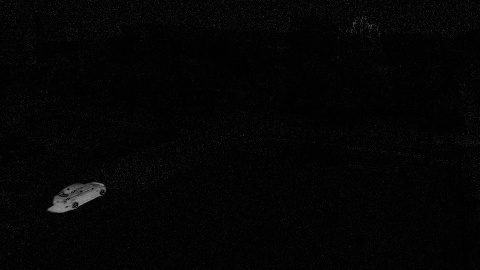} \vspace{1 mm}
        \includegraphics[width=\columnwidth]{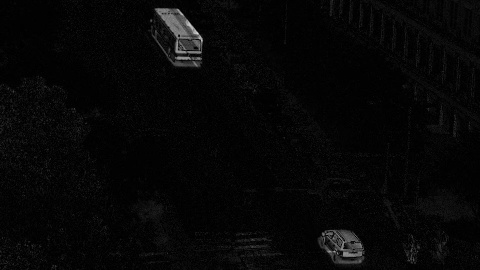} \vspace{1 mm}
        \includegraphics[width=\columnwidth]{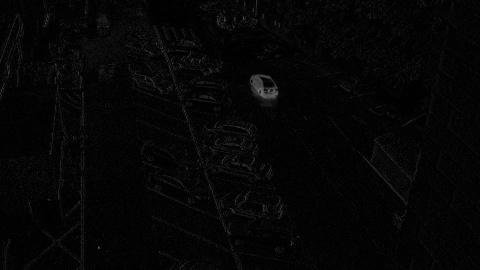} 
    \end{minipage}
    \hfill
    \begin{minipage}[b]{0.19\columnwidth}
        \centering
        \phantom{blank} \vspace{1 mm}
        \includegraphics[width=\columnwidth]{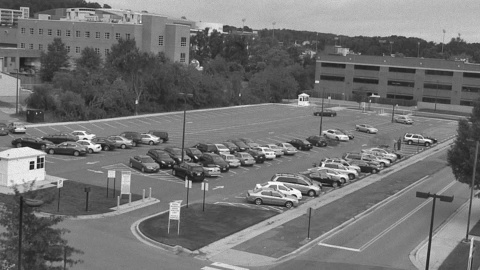} \vspace{1 mm}
        \includegraphics[width=\columnwidth]{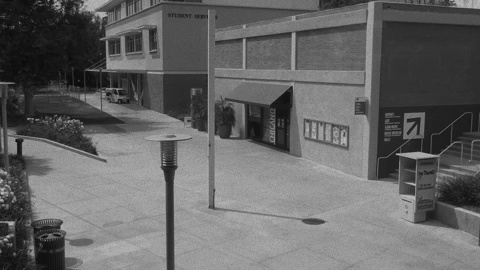} \vspace{1 mm}
        \includegraphics[width=\columnwidth]{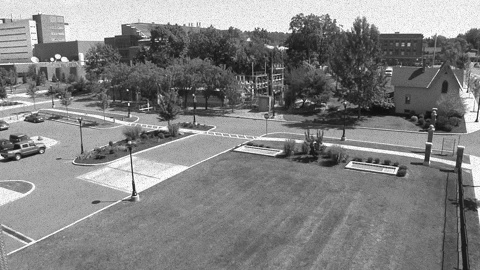} \vspace{1 mm}
        \includegraphics[width=\columnwidth]{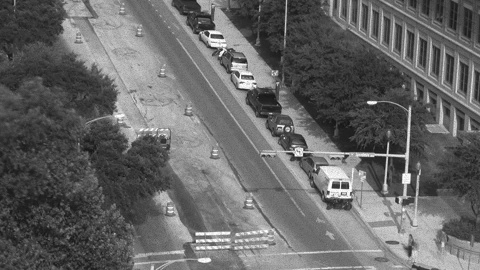} \vspace{1 mm}
        \includegraphics[width=\columnwidth]{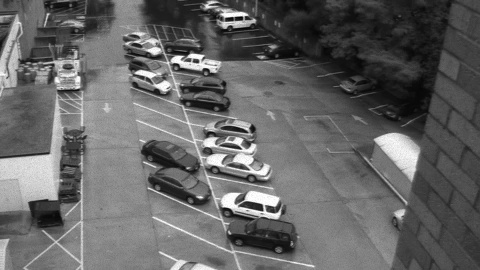} 
    \end{minipage}
    \hfill
    \begin{minipage}[b]{0.19\columnwidth}
        \centering
        \phantom{blank} \vspace{1 mm}
        \includegraphics[width=\columnwidth]{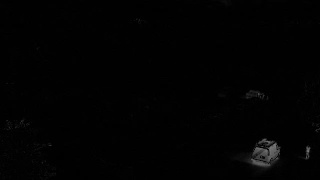} \vspace{1 mm}
        \includegraphics[width=\columnwidth]{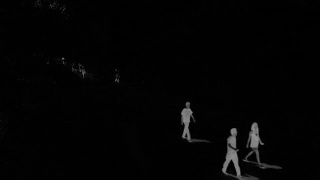} \vspace{1 mm}
        \includegraphics[width=\columnwidth]{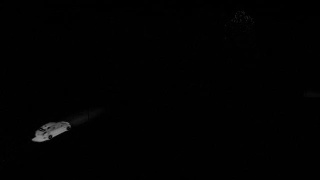} \vspace{1 mm}
        \includegraphics[width=\columnwidth]{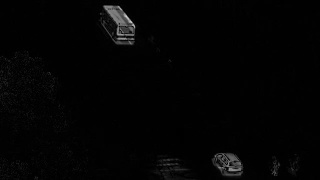} \vspace{1 mm}
        \includegraphics[width=\columnwidth]{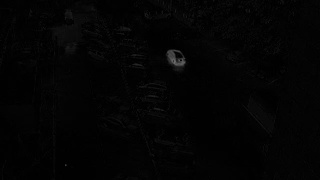} 
    \end{minipage}
    \hfill
    \begin{minipage}[b]{0.19\columnwidth}
        \centering
        \phantom{blank} \vspace{1 mm}
        \includegraphics[width=\columnwidth]{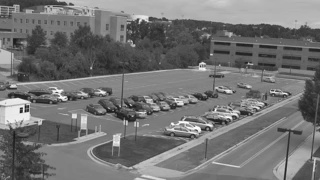} \vspace{1 mm}
        \includegraphics[width=\columnwidth]{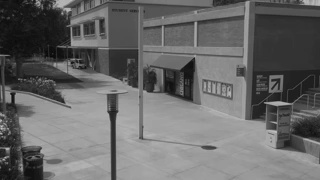} \vspace{1 mm}
        \includegraphics[width=\columnwidth]{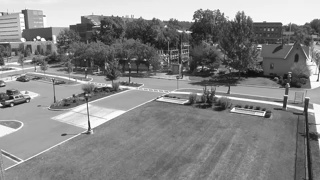} \vspace{1 mm}
        \includegraphics[width=\columnwidth]{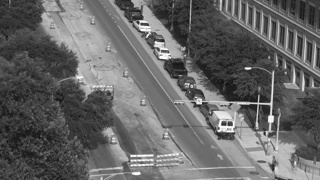} \vspace{1 mm}
        \includegraphics[width=\columnwidth]{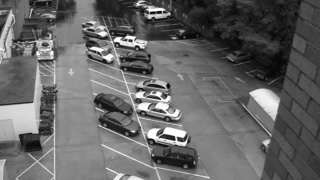} 
    \end{minipage}
    
    \vspace{-0.08in}
    \caption{Visual results for video background subtraction. The first column is the original frames. The second and third columns are the foreground and background output by LRMC with $p = 10\%$. The fourth and fifth columns are the outputs by LRMC with $p = 100\%$.}    \label{fig:visualization_virat}
\end{figure}

\vspace{0.05in}
\noindent\textbf{Ultrasound imaging.} 
We compare LRMC with CORONA \cite{cohen2019deep}, a state-of-the-art learning-based RPCA method (i.e., RMC with $p=1$), on ultrasound imaging benchmark\footnote{Available at \url{https://www.wisdom.weizmann.ac.il/~yonina}.} that consists of $2400$ training examples and $800$ validation examples. Each example is of size $1024 \times 40$. The target rank of the low-rank matrix is set to be $r=1$. Recovery accuracy is measured by 
$\mathrm{loss} = \mathrm{MSE}(X_\mathrm{output}, X_\star)$, where $\mathrm{MSE}$ stands for mean square error. The test results are reported in \Cref{tab:ultrasound}.

Note that CORONA is specifically designed for fully observed ultrasound imaging. Thus, it is not a surprise that our recovery accuracy is slightly worse than CORONA. However, the average runtime of LRMC is substantially faster than CORONA. We believe that our speed advantage will be even more significant on larger-scale examples, which is indeed one of our main contributions: scalability. Using partial observation will further improve the scalability of LRMC; however, the advantage is not obvious on this relatively small dataset.

\begin{table}[h]
\centering
\caption{Runtime and loss results for ultrasound imaging. All algorithms halt when $\|\X_{k}-\BX_{k-1}\|_\fro/\|\BX_{k-1}\|_\fro <10^{-4}$ or after 32 iterations. } \label{tab:ultrasound}
\begin{tabular}{l|c|c}
\toprule
\textsc{Algorithm} & \textsc{Average inference time} & \textsc{Average} $\mathrm{loss}$ \cr
\midrule
LRMC ($p = 10\%$) & \textbf{0.0036} secs & $4.68\times 10^{-3}$ \cr
LRMC ($p = 100\%$)   & \underline{0.0056} secs & \underline{$9.97\times 10^{-4}$}       \cr
CORONA   & 0.9225 secs  & \textbf{4.88}\,$\mathbf{\times\,10^{-4}}$       \cr
\bottomrule
\end{tabular}
\end{table}

\vspace{0.05in}
\noindent\textbf{Face modeling.} 
We use AR Face Database \cite{martinez1998ar} as the face modeling benchmark, which consists of 3,276 images of 126 human subjects captured under diverse conditions. Each human subject has 26 frontal-view images encompassing varying facial expressions, illumination settings, and accessories such as sunglasses and scarves. The face images are vectorized, and those belonging to the same subject are grouped into one data matrix. For our experiments, subjects 1–10 are designated as the test set, while the remaining subjects form the training set. To enhance training, we employ various image augmentation techniques, including flipping, rotation, scaling, shearing, shifting, and cropping to create more training samples.

The runtime test results are summarized in Table~\ref{tab:face}, and selected visual results are presented in Fig.~\ref{fig:visualization_faces}. One can observe that LRMC dominates the runtime comparison. In terms of visual results, LRMC and ScaledGD provide similar facial recovery, though the faces recovered with full observations are less noisy than at $30\%$ observation rate. 

\newlength{\faceimgwidth}
\setlength{\faceimgwidth}{0.10\columnwidth} 

\begin{figure}[H]
    \centering   
    \begin{minipage}[b]{\faceimgwidth}
      \centering
      \phantom{blank} \vspface
        \includegraphics[width=\columnwidth]{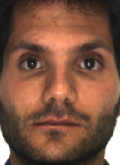} \vspface
        \includegraphics[width=\columnwidth]{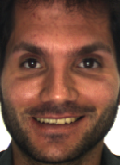} \vspface
        \includegraphics[width=\columnwidth]{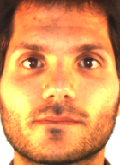} \vspface
        \includegraphics[width=\columnwidth]{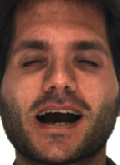} \vspface
        \includegraphics[width=\columnwidth]{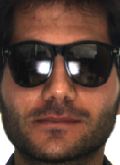} \vspface
        \includegraphics[width=\columnwidth]{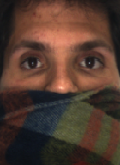} \vspface        
    \end{minipage}
    \begin{minipage}[b]{\faceimgwidth}
        \centering
        \phantom{blank} \vspface
        \includegraphics[width=\columnwidth]{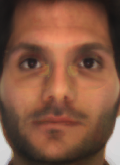} \vspface
        \includegraphics[width=\columnwidth]{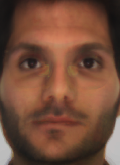} \vspface
        \includegraphics[width=\columnwidth]{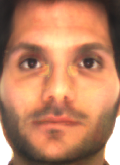} \vspface
        \includegraphics[width=\columnwidth]{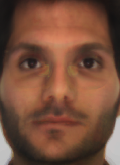} \vspface
        \includegraphics[width=\columnwidth]{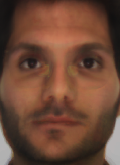} \vspface
        \includegraphics[width=\columnwidth]{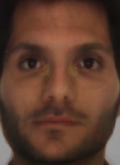} \vspface
    \end{minipage}
    \hfill
    \begin{minipage}[b]{\faceimgwidth}
        \centering
        \phantom{blank} \vspface
        \includegraphics[width=\columnwidth]{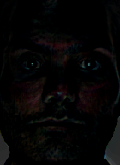} \vspface
        \includegraphics[width=\columnwidth]{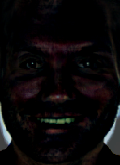} \vspface
        \includegraphics[width=\columnwidth]{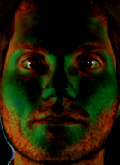} \vspface
        \includegraphics[width=\columnwidth]{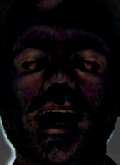} \vspface
        \includegraphics[width=\columnwidth]{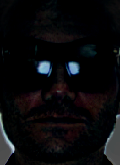} \vspface
        \includegraphics[width=\columnwidth]{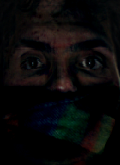} \vspface
    \end{minipage}
    \hfill
    \hfill
    \begin{minipage}[b]{\faceimgwidth}
        \centering
        \phantom{blank} \vspface
        \includegraphics[width=\columnwidth]{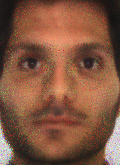} \vspface
        \includegraphics[width=\columnwidth]{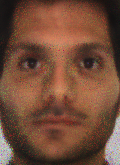} \vspface
        \includegraphics[width=\columnwidth]{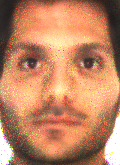} \vspface
        \includegraphics[width=\columnwidth]{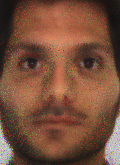} \vspface
        \includegraphics[width=\columnwidth]{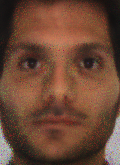} \vspface
        \includegraphics[width=\columnwidth]{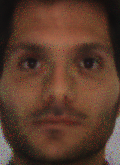} \vspface
    \end{minipage}
    \hfill
    \begin{minipage}[b]{\faceimgwidth}
        \centering
        \phantom{blank} \vspface
        \includegraphics[width=\columnwidth]{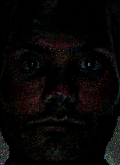} \vspface
        \includegraphics[width=\columnwidth]{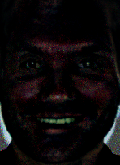} \vspface
        \includegraphics[width=\columnwidth]{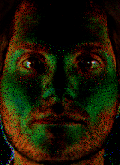} \vspface
        \includegraphics[width=\columnwidth]{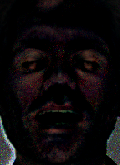} \vspface
        \includegraphics[width=\columnwidth]{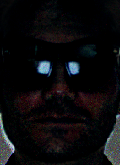} \vspface
        \includegraphics[width=\columnwidth]{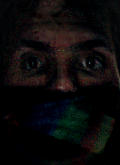} \vspface
    \end{minipage}
    \hfill
    \begin{minipage}[b]{\faceimgwidth}
        \centering
        \phantom{blank} \vspface
        \includegraphics[width=\columnwidth]{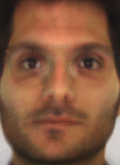} \vspface
        \includegraphics[width=\columnwidth]{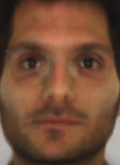} \vspface
        \includegraphics[width=\columnwidth]{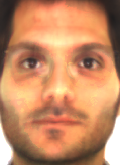} \vspface
        \includegraphics[width=\columnwidth]{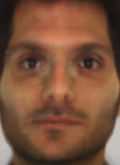} \vspface
        \includegraphics[width=\columnwidth]{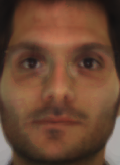} \vspface
        \includegraphics[width=\columnwidth]{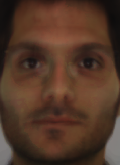} \vspface
    \end{minipage}
    \hfill
    \begin{minipage}[b]{\faceimgwidth}
        \centering
        \phantom{blank} \vspface
        \includegraphics[width=\columnwidth]{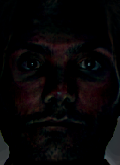} \vspface
        \includegraphics[width=\columnwidth]{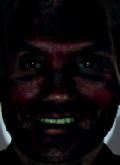} \vspface
        \includegraphics[width=\columnwidth]{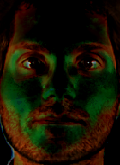} \vspface
        \includegraphics[width=\columnwidth]{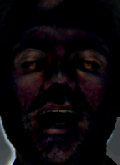} \vspface
        \includegraphics[width=\columnwidth]{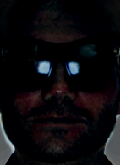} \vspface
        \includegraphics[width=\columnwidth]{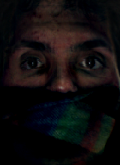} \vspface
    \end{minipage}
    \hfill
    \begin{minipage}[b]{\faceimgwidth}
        \centering
        \phantom{blank} \vspface
        \includegraphics[width=\columnwidth]{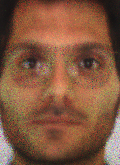} \vspface
        \includegraphics[width=\columnwidth]{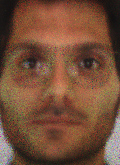} \vspface
        \includegraphics[width=\columnwidth]{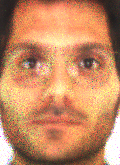} \vspface
        \includegraphics[width=\columnwidth]{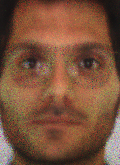} \vspface
        \includegraphics[width=\columnwidth]{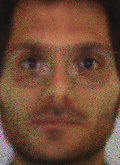} \vspface
        \includegraphics[width=\columnwidth]{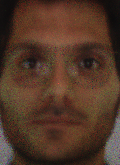} \vspface
    \end{minipage}
    \hfill
    \begin{minipage}[b]{\faceimgwidth}
        \centering
        \phantom{blank} \vspface
        \includegraphics[width=\columnwidth]{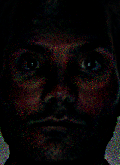} \vspface
        \includegraphics[width=\columnwidth]{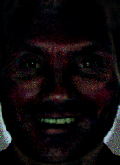} \vspface
        \includegraphics[width=\columnwidth]{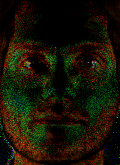} \vspface
        \includegraphics[width=\columnwidth]{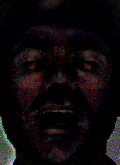} \vspface
        \includegraphics[width=\columnwidth]{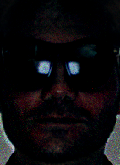} \vspface
        \includegraphics[width=\columnwidth]{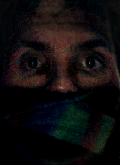} \vspface
    \end{minipage}

    \vspace{-0.2in}
    \caption{Visual results for face modeling. The first column shows the original images. The second and third columns are the recovered faces and removed outliers output by LRMC with $p = 100\%$. The fourth and fifth columns are the outputs by LRMC with $p = 30\%$. The sixth and seventh columns are the outputs by ScaledGD with $p = 100\%$, and the eighth and ninth columns are the outputs by ScaledGD with $p = 30\%$.}
\label{fig:visualization_faces}

\end{figure}

\begin{table}[h]
\centering
\caption{Runtime results for face modeling. All algorithms halt when $\|\X_{k}-\BX_{k-1}\|_\fro/\|\BX_{k-1}\|_\fro<10^{-4}$ or after 32 iterations.}  \label{tab:face}
\begin{tabular}{l|c@{\hskip 0pt}}
\toprule
\textsc{Algorithm} & \textsc{Average inference time}  \cr
\midrule
LRMC ($p = 100\%$) & \underline{0.29} secs  \cr
LRMC ($p = 30\%$)   & \textbf{0.16} secs  \cr
ScaledGD ($p = 100\%$)   & 3.26 secs   \cr
ScaledGD ($p = 30\%$)  & 1.97 secs \cr
\bottomrule
\end{tabular}
\end{table}

\newlength{\cloudimgwidth}
\setlength{\cloudimgwidth}{0.22\columnwidth} 

\begin{figure*}[ht]
    \centering
    \begin{minipage}{\cloudimgwidth}
        \includegraphics[width=\columnwidth]{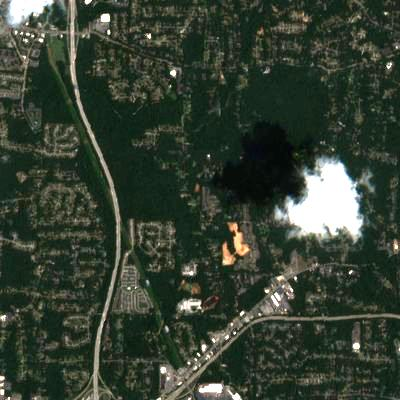}\vspace{1 mm}
        \includegraphics[width=\columnwidth]{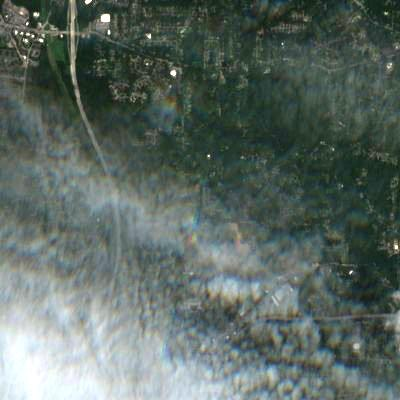}\vspace{1 mm} 
        \includegraphics[width=\columnwidth]{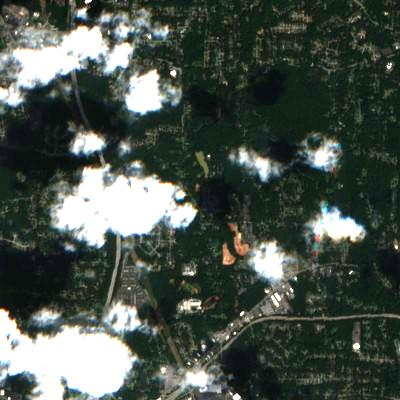}\vspace{1 mm}
    \end{minipage}%
    \hfill
    \begin{minipage}{\cloudimgwidth}
        \includegraphics[width=\columnwidth]{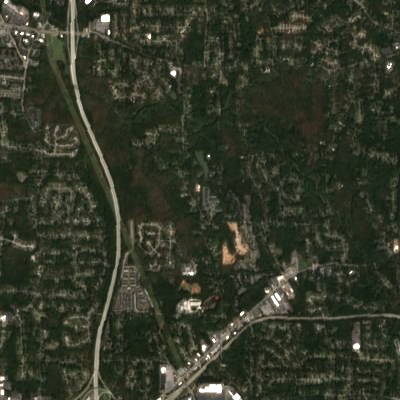}\vspace{1 mm}
        \includegraphics[width=\columnwidth]{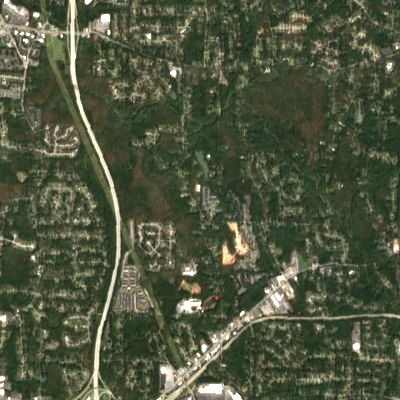}\vspace{1 mm} 
        \includegraphics[width=\columnwidth]{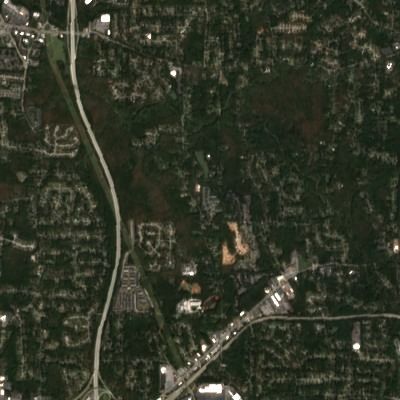}\vspace{1 mm}
    \end{minipage}%
    \hfill
    \begin{minipage}{\cloudimgwidth}
        \includegraphics[width=\columnwidth]{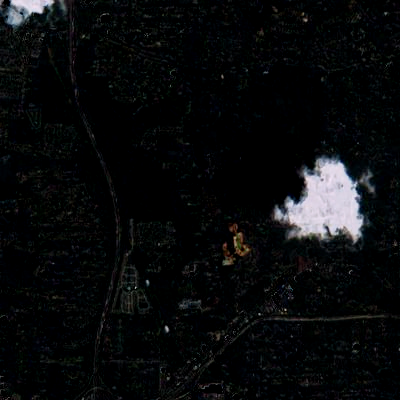}\vspace{1 mm}
        \includegraphics[width=\columnwidth]{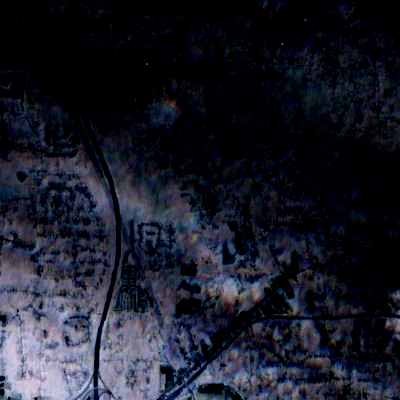}\vspace{1 mm} 
        \includegraphics[width=\columnwidth]{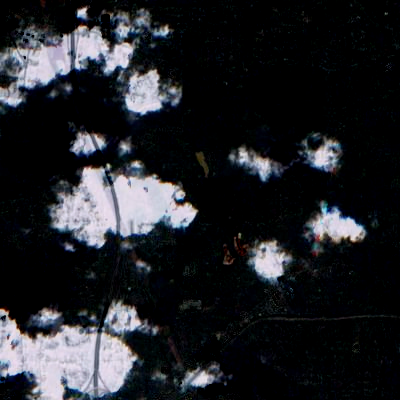}\vspace{1 mm}
    \end{minipage}%
    \hfill
    \begin{minipage}{\cloudimgwidth}
        \includegraphics[width=\columnwidth]{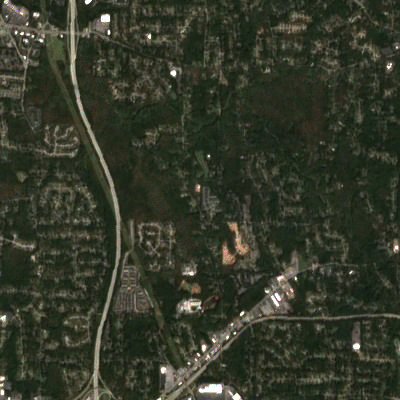}\vspace{1 mm}
        \includegraphics[width=\columnwidth]{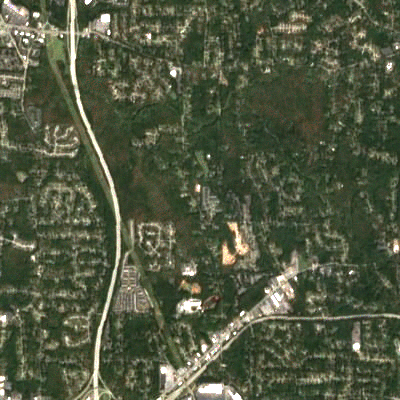}\vspace{1 mm} 
        \includegraphics[width=\columnwidth]{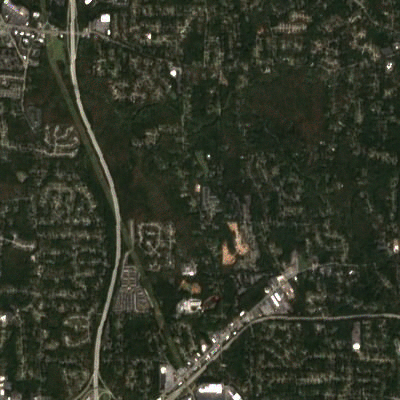}\vspace{1 mm}
    \end{minipage}%
    \hfill
    \begin{minipage}{\cloudimgwidth}
        \includegraphics[width=\columnwidth]{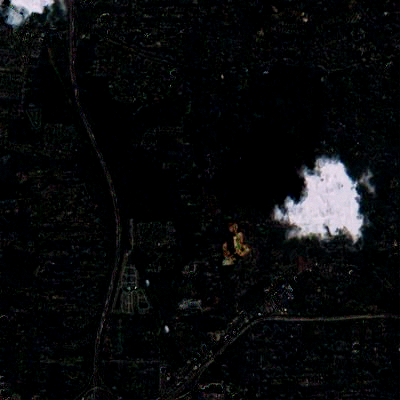}\vspace{1 mm}
        \includegraphics[width=\columnwidth]{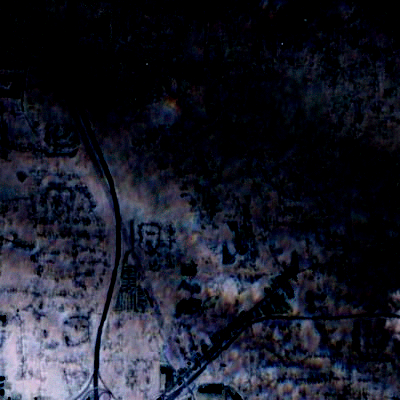}\vspace{1 mm} 
        \includegraphics[width=\columnwidth]{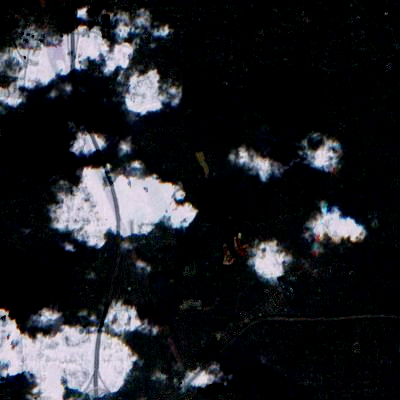}\vspace{1 mm}
    \end{minipage}%
    \hfill
    \begin{minipage}{\cloudimgwidth}
        \includegraphics[width=\columnwidth]{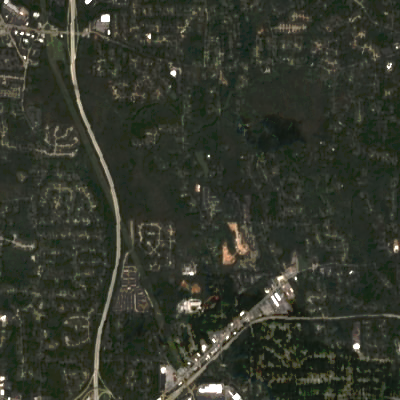}\vspace{1 mm}
        \includegraphics[width=\columnwidth]{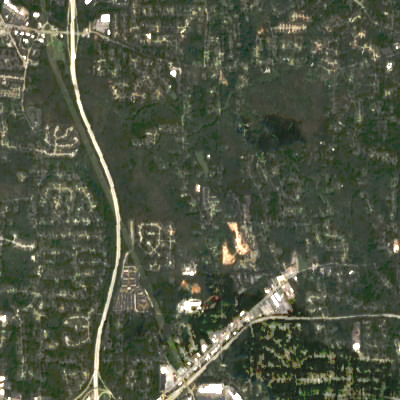}\vspace{1 mm} 
        \includegraphics[width=\columnwidth]{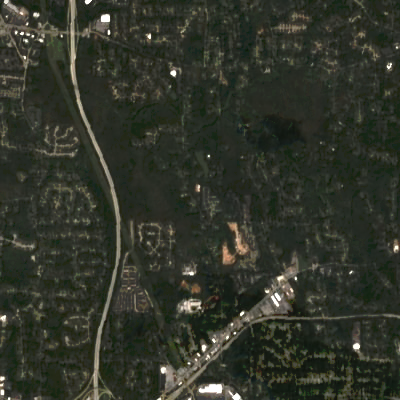}\vspace{1 mm}
    \end{minipage}%
    \hfill
    \begin{minipage}{\cloudimgwidth}
        \includegraphics[width=\columnwidth]{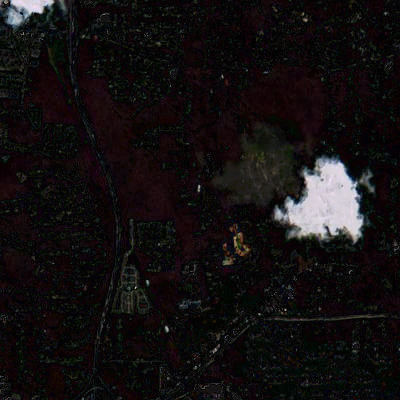}\vspace{1 mm}
        \includegraphics[width=\columnwidth]{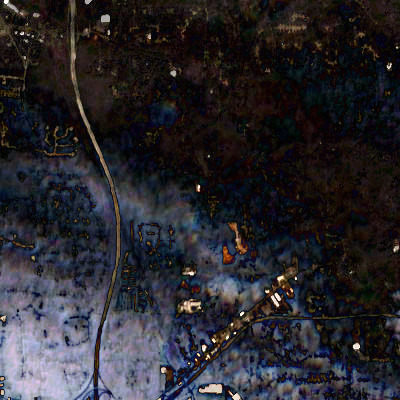}\vspace{1 mm} 
        \includegraphics[width=\columnwidth]{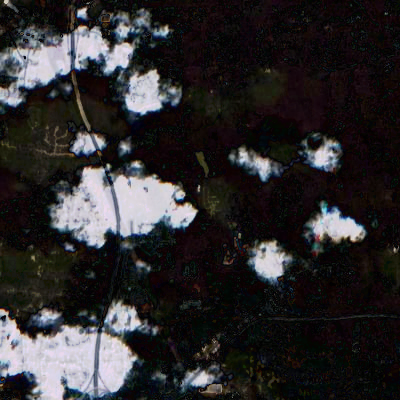}\vspace{1 mm}
    \end{minipage}%
    \hfill
    \begin{minipage}{\cloudimgwidth}
        \includegraphics[width=\columnwidth]{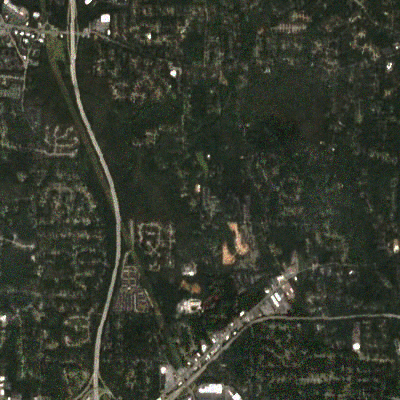}\vspace{1 mm}
        \includegraphics[width=\columnwidth]{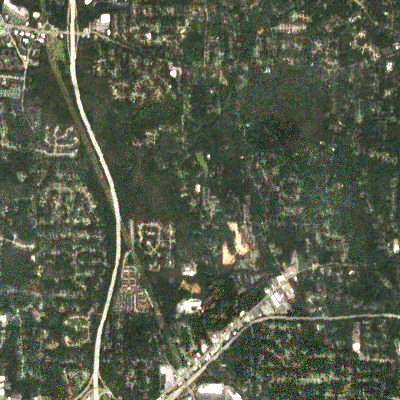}\vspace{1 mm} 
        \includegraphics[width=\columnwidth]{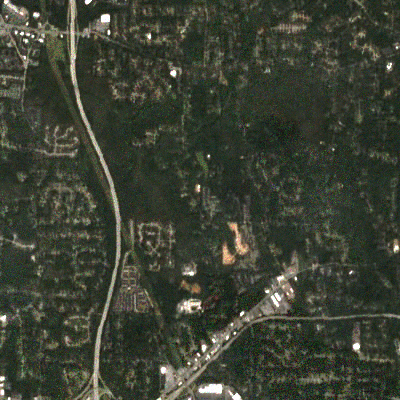}\vspace{1 mm}
    \end{minipage}%
    \hfill
    \begin{minipage}{\cloudimgwidth}
        \includegraphics[width=\columnwidth]{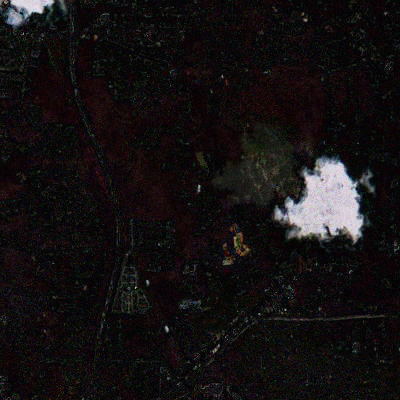}\vspace{1 mm}
        \includegraphics[width=\columnwidth]{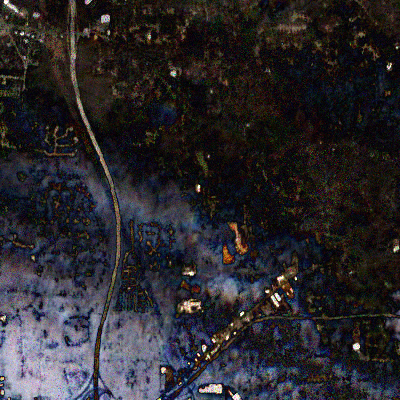}\vspace{1 mm} 
        \includegraphics[width=\columnwidth]{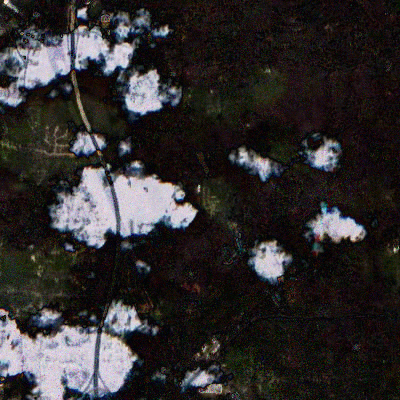}\vspace{1 mm}
    \end{minipage}%

    \vspace{-0.05in}
    \caption{Visual results for cloud removal at $400\times400$ resolution. The first column is the observed satellite image blocks from Atlanta City on different dates. The second and third columns show the recovered cloud-free ground and cloud images output by LRMC with $p = 100\%$. The fourth and fifth columns are the outputs by LRMC with $p = 30\%$. The sixth and seventh columns are the output by ScaledGD with $p = 100\%$. The eighth and ninth columns are the outputs by ScaledGD with $p = 30\%$.}
    \label{fig:cloud_removal}
\end{figure*}

\vspace{0.05in}
\noindent\textbf{Cloud removal from multi-temporal satellite imagery.} 
We evaluate the scalability of LRMC to the large-scale task of cloud removal from high-resolution multi-temporal satellite imagery. The dataset used for this experiment is curated using Sentinel-2 instruments (A level 1-A top-of-atmosphere reflectance product) from the European Space Agency's Copernicus mission. The data is acquired through the Sentinel Hub API\footnote{Data accessed through the Sentinel Hub API, provided by Sinergise Solutions d.o.o., \url{https://www.sentinel-hub.com/}.}. The dataset includes multiple blocks from each of the 100 most populous U.S. cities, comprising 60 images captured between January 2021 and June 2024. To rigorously test scalability, we perform experiments at three resolutions: $400 \times 400$ pixels ($4.8 \times 10^{5}$ rows), $1000 \times 1000$ pixels ($3 \times 10^{6}$ rows), and $2000 \times 2000$ pixels ($1.2 \times 10^7$ rows) per image.

Table~\ref{tab:cloud2} outlines the average runtime and reconstruction loss comparisons between LRMC and the baseline ScaledGD. At lower resolutions, i.e., $400\times400$ and $1000\times1000$, LRMC is significantly faster than ScaledGD while maintaining a lower loss. Moreover, at the high resolution of $2000\times2000$, the full observation cases lead to memory exhaustion for both algorithms. Thus, it is crucial to have the ability to work with partial observations for the sake of memory efficiency for the proposed learning pipeline. With 30\% observations, LRMC demonstrates superior efficiency: It recovers the background around 10$\times$ faster than ScaledGD while achieving lower reconstruction loss. This confirms that LRMC offers superior scalability for large-scale real-world tasks. Selected visual results are presented in Fig.~\ref{fig:cloud_removal}, where both algorithms provide visually crispy reconstruction, under full or partial observations.

\begin{table*}[!ht]
\centering
\caption{Comparison of average runtime and reconstruction loss for cloud removal task at varying resolutions. ``\ding{56}" denotes out-of-memory failure. All algorithms halt when $\|\X_{k}-\BX_{k-1}\|_\fro/\|\BX_{k-1}\|_\fro<10^{-4}$ or after 32 iterations.} \label{tab:cloud2}
\setlength{\tabcolsep}{6pt} 
\begin{tabular}{l|cc|cc|cc}
\toprule
& \multicolumn{2}{c|}{$\mathbf{400\times 400}$} 
& \multicolumn{2}{c|}{$\mathbf{1000\times 1000}$} 
& \multicolumn{2}{c}{$\mathbf{2000\times 2000}$} \\

\textsc{Algorithm} 
& \textsc{Runtime} & \textsc{Loss} 
& \textsc{Runtime} & \textsc{Loss} 
& \textsc{Runtime} & \textsc{Loss} \\
\midrule

LRMC ($p = 100\%$) 
& \underline{3.27} secs  & \textbf{9.42}\,$\mathbf{\times\,10^{-5}}$ 
& \underline{26.31} secs & \textbf{9.63}\,$\mathbf{\times\,10^{-5}}$ 
& \ding{56} & \ding{56} \\

LRMC ($p = 30\%$) 
& \textbf{1.65} secs & \underline{5.22 $\times 10^{-4}$} 
& \textbf{16.07} secs& \underline{5.73 $\times 10^{-4}$} 
& \textbf{68.24} secs & \textbf{9.81} $\mathbf{\times 10^{-4}}$ \\

ScaledGD ($p = 100\%$) 
& 48.61 secs & $6.88\times 10^{-4}$ 
& 198.89 secs & $7.89\times 10^{-4}$ 
& \ding{56} & \ding{56}   \\

ScaledGD ($p = 30\%$)  
& 28.54 secs & $4.46\times 10^{-3}$
& 104.39 secs &  $5.78\times 10^{-3}$ 
& \underline{536.09} secs & \underline{7.92 $\times 10^{-3}$}  \\
\bottomrule
\end{tabular}
\end{table*}

\section{Proofs} \label{sec:proofs}
In this section, we provide the mathematical proofs for the claimed theoretical results. Note that the proof of our convergence theorem follows the route established in \cite{tong2021accelerating}. However, the details of our proof are quite involved since we replaced the sparsification operator, which substantially changes the method of outlier detection.

Let $\BL_\star:=\BU_\star\BSigma_\star^{1/2}$ and $\BR_\star:=\BV_\star\BSigma_\star^{1/2}$ where $\BU_\star\BSigma_\star\BV_\star^\top$ is the compact SVD of $\BX_\star$. 
For theoretical analysis, we consider the error metric for decomposed rank-$r$ matrices:
\[
    &~ \dist^2(\BL,\BR;\BL_\star,\BR_\star) \cr 
    := & \inf_{\substack{\BQ\in\mathbb{R}^{r\times r},\\\rank(\BQ)=r} } \|(\BL\BQ-\BL_\star)\BSigma_\star^{1/2}\|_\fro^2  + \| (\BR\BQ^{-\top}-\BR_\star)\BSigma_\star^{1/2} \|_\fro^2 .
\]
By \cite[Lemma~9]{tong2021accelerating}, the optimal alignment $\BQ$ exists and is invertible if $[\BL,\BR]$ is sufficiently close to $[\BL_\star,\BR_\star]$, specifically $
    \dist(\BL,\BR;\BL_\star,\BR_\star) < c \sigma_r(\BX_\star)
$
for some $0< c<1$.

\setcounter{lemma}{5}



For ease of presentation, we take $n:=n_1=n_2$ in the rest of this section, but we emphasize that similar results can be easily drawn for the rectangular matrix setting. Furthermore, we introduce following shorthand for notational convenience: $\BQ_k$ denotes the optimal alignment matrix between $(\BL_k, \BR_k)$ and $(\BL_\star, \BR_\star)$, $\BL_\natural:=\BL_k\BQ_k$, $\BR_\natural:=\BR_k\BQ_k^{-\top}$, $\BDelta_L:=\BL_\natural-\BL_\star$, $\BDelta_R:=\BR_\natural-\BR_\star$, and $\BDelta_S:=\BS_{k+1}-\BS_\star$.

\vspace{-0.1in}
\subsection{Proof of Theorem~\ref{thm:main_theorem}}
We first present the theorems of local linear convergence and guaranteed initialization. The proofs of these two theorems can be found in Sections~\ref{sec:local conv} and \ref{sec:guaranteed initialization}, respectively. 

\begin{theorem}[Local linear convergence] \label{thm:local convergence}
Suppose that $\BX_\star=\BL_\star\BR_\star^\top$ is a rank-$r$ matrix with $\mu$-incoherence and $\BS_\star$ is an $\alpha$-sparse matrix with $\alpha\leq\frac{1}{10^4\mu r^{1.5}}$. Let $\BQ_k$ be the optimal alignment matrix between $[\BL_k,\BR_k]$ and $[\BL_\star,\BR_\star]$. If the initial guesses obey the conditions
\[
&~ \distzero \leq \varepsilon_0 \sigma_r(\BX_\star), \cr
&~ \|(\BL_0\BQ_0-\BL_\star)\BSigma_\star^{1/2}\|_{2,\infty}  \lor \|(\BR_0\BQ_0^{-\top}-\BR_\star)\BSigma_\star^{1/2}\|_{2,\infty} \\ 
\leq&~\sqrt{{\mu r}/{n}} \sigma_r(\BX_\star)
\]
with $\varepsilon_0:=0.02$, then by setting the thresholding values $\zeta_k=\|\BX_\star-\BL_{k-1}\BR_{k-1}^\top\|_\infty$ and the fixed step size $\eta_k=\eta\in[\frac{1}{4},\frac{8}{9}]$, the iterates of Algorithm~\ref{algo:LRMC} satisfy
\[
& \distk\leq\varepsilon_0 \tau^k \sigma_r(\BX_\star), \\ 
& \|(\BL_k\BQ_k-\BL_\star)\BSigma_\star^{1/2}\|_{2,\infty}  \lor \|(\BR_k\BQ_k^{-\top}-\BR_\star)\BSigma_\star^{1/2}\|_{2,\infty}  \\ 
 \leq&~\sqrt{{\mu r}/{n}} \tau^k \sigma_r(\BX_\star),
\]
where the convergence rate $\tau:=1-0.6\eta$.
\end{theorem}

\begin{theorem}[Guaranteed initialization] \label{thm:initial}
Suppose that $\BX_\star=\BL_\star\BR_\star^\top$ is a rank-$r$ matrix with $\mu$-incoherence and $\BS_\star$ is an $\alpha$-sparse matrix with $\alpha\leq\frac{c_0}{\mu r^{1.5}\kappa}$ for some small positive constant $c_0\leq\frac{1}{35}$. Let $\BQ_0$ be the optimal alignment matrix between $[\BL_0,\BR_0]$ and $[\BL_\star,\BR_\star]$. By setting the thresholding values $\zeta_0=\|\BX_\star\|_\infty$, the initial guesses satisfy
\[
& \distzero \leq 10c_0 \sigma_r(\BX_\star), \\ 
& \|(\BL_0\BQ_0-\BL_\star)\BSigma_\star^{1/2}\|_{2,\infty}  \lor \|(\BR_0\BQ_0^{-\top}-\BR_\star)\BSigma_\star^{1/2}\|_{2,\infty} \\ 
 \leq&~\sqrt{{\mu r}/{n}} \sigma_r(\BX_\star) .
\]
\end{theorem}

In addition, we present a lemma that verifies our selection of thresholding values is indeed effective.

\begin{lemma} \label{lm:sparity}
At the $(k+1)$-th iteration of Algorithm~\ref{algo:LRMC}, taking the thresholding value $\zeta_{k+1}:=\|\BX_\star-\BX_{k}\|_\infty$ gives
\[
     \|\BS_\star-\BS_{k+1}\|_\infty \leq 2 \|\BX_\star-\BX_{k}\|_\infty,\, \supp(\BS_{k+1})\subseteq \supp(\BS_\star).
\]
\end{lemma}
\begin{proof}
Denote $\Omega_\star:=\supp(\BS_\star)$ and $\Omega_{k+1}:=\supp(\BS_{k+1})$. 
Recall that $\BS_{k+1} = \cS_{\zeta_{k+1}}(\BY-\BX_{k})=\cS_{\zeta_{k+1}}(\BS_\star+\BX_\star-\BX_{k})$. Since $[\BS_\star]_{i,j} = 0$ outside its support, so $[\BY-\BX_{k}]_{i,j} = [\BX_\star-\BX_{k}]_{i,j}$ for the entries $(i,j)\in\Omega_\star^c$. Applying the chosen thresholding value $\zeta_{k+1}:=\|\BX_\star-\BX_{k}\|_\infty$, one have $[\BS_{k+1}]_{i,j}=0$ for all $(i,j)\in \Omega_\star^c$. Hence, the support of $\BS_{k+1}$ must belongs to the support of $\BS_\star$, i.e.,
\[
\supp(\BS_{k+1})=\Omega_{k+1}\subseteq\Omega_\star=\supp(\BS_\star).
\]
This proves our first claim.

Obviously, $[\BS_\star-\BS_{k+1}]_{i,j}=0$ for all $(i,j)\in\Omega_\star^c$. Moreover, we can split the entries in $\Omega_\star$ into two groups:
\[
    \Omega_{k+1} &= \{(i,j)~|~ |[\BY-\BX_k]_{i,j}|>\zeta_{k+1} \textnormal{ and } [\BS_\star]_{i,j}\neq 0\}, \cr
    \Omega_\star\backslash\Omega_{k+1} &= \{(i,j)~|~ |[\BY-\BX_k]_{i,j}|\leq\zeta_{k+1} \textnormal{ and } [\BS_\star]_{i,j}\neq 0\},
\]
and it holds
\[
|[\BS_\star-\BS_{k+1}]_{i,j}|=& 
\begin{cases}
|[\BX_{k}-\BX_\star]_{i,j} - \mathrm{sign}([\BY-\BX_{k}]_{i,j})\zeta_{k+1} |&   \cr
|[\BS_\star]_{i,j}|         & \cr
\end{cases}  \cr
\leq&
\begin{cases}
|[\BX_{k}-\BX_\star]_{i,j}|+\zeta_{k+1}      &  \cr
|[\BX_\star-\BX_{k}]_{i,j}|+\zeta_{k+1}  & \cr
\end{cases} \cr
\leq&
\begin{cases}
2\|\BX_\star-\BX_{k}\|_\infty      & \quad~~ (i,j)\in \Omega_{k+1};   \cr
2\|\BX_\star-\BX_{k}\|_\infty      & \quad~~ (i,j)\in \Omega_\star\backslash\Omega_{k+1}. \cr
\end{cases}
\]
Therefore, it concludes $\|\BS_\star-\BS_{k+1}\|_\infty\leq 2\|\BX_\star-\BX_{k}\|_\infty $.
\end{proof}

Now, we are ready to prove Theorem~\ref{thm:main_theorem}.

\begin{proof}[Proof of Theorem~~\ref{thm:main_theorem}]
Take $c_0=10^{-4}$ in Theorem~\ref{thm:initial}. Thus, the results of Theorem~\ref{thm:initial} satisfy the condition of Theorem~\ref{thm:local convergence}, and gives 
\[
    & \distk\leq 0.02 (1-0.6\eta)^k \sigma_r(\BX_\star), \\
    & \|(\BL_k\BQ_k-\BL_\star)\BSigma_\star^{1/2}\|_{2,\infty}  \lor \|(\BR_k\BQ_k^{-\top}-\BR_\star)\BSigma_\star^{1/2}\|_{2,\infty} \\
     \leq&~\sqrt{{\mu r}/{n}} (1-0.6\eta)^k \sigma_r(\BX_\star)
\]
for all $k\geq0$. \cite[Lemma~3]{tong2021accelerating} states that
\[
    \| \BL_k\BR_k^\top - \BX_\star \|_\fro \leq 1.5~\distk
\]
as long as $\|(\BL_k\BQ_k-\BL_\star)\BSigma_\star^{1/2}\|_{2,\infty}  \lor \|(\BR_k\BQ_k^{-\top}-\BR_\star)\BSigma_\star^{1/2}\|_{2,\infty} \leq\sqrt{{\mu r}/{n}} \sigma_r(\BX_\star)$.
The first claim is proved.

When $k\geq 1$, the second claim is directly followed by Lemma~\ref{lm:sparity}. When $k=0$, take $\BX_{-1}=\bm{0}$, then one can see $\BS_0=\cS_{\zeta_0}(\BY)=\cS_{\zeta_0}(\BY-\BX_{-1})$, where $\zeta_0=\|\BX_\star\|_\infty=\|\BX_\star-\BX_{-1}\|_\infty$. Applying Lemma~\ref{lm:sparity} again, we have the second claim for all $k\geq0$. 
This finishes the proof.
\end{proof}

\vspace{-0.1in}
\subsection{Auxiliary lemmas}
Before we can present the proofs for Theorems~\ref{thm:local convergence} and \ref{thm:initial}, several important auxiliary lemmas must be processed.

\begin{lemma} \label{lm:bound of sparse matrix}
If $\BS\in \mathbb{R}^{n \times n}$ is $\alpha$-sparse, then it holds
$
    \|\BS\|_2\leq\alpha n \|\BS\|_\infty$,  
$    \|\BS\|_{2,\infty}\leq\sqrt{\alpha n} \|\BS\|_\infty$, and $\|\BS\|_{1,\infty}\leq \alpha n \|\BS\|_\infty.
$
\end{lemma}
\begin{proof}
The first claim has been shown as {\cite[Lemma~4]{netrapalli2014non}}. The rest two claims are directly followed by the fact $\BS$ has at most $\alpha n$ non-zero elements in each row and each column.
\end{proof}

\begin{lemma} \label{lm:Delta_F_norm}
If 
$
    \distk \leq \varepsilon_0\tau^k \sigma_r(\BX_\star),
$ 
then it holds 
$
\| \BDelta_L\BSigma_\star^{1/2} \|_\fro \lor \| \BDelta_R\BSigma_\star^{1/2} \|_\fro \leq \varepsilon_0\tau^{k}\sigma_r(\BX_\star) $ and $
\| \BDelta_L\BSigma_\star^{1/2} \|_2 \lor \| \BDelta_R\BSigma_\star^{1/2} \|_2 \leq \varepsilon_0\tau^{k}\sigma_r(\BX_\star).
$
\end{lemma}
\begin{proof}
The first claim is directly followed by the definition of $\dist$. 
By the fact that $\|\bm{A}\|_2\leq\|\bm{A}\|_\fro$ for any matrix, we deduce the second claim from the first claim.
\end{proof}


\begin{lemma} \label{lm:L_R_scale_sigma_half}  
If 
$
    \distk \leq \varepsilon_0\tau^k \sigma_r(\BX_\star),
$ 
then it holds
$
    \|\BL_\natural(\BL_\natural^\top\BL_\natural)^{-1}\BSigma_\star^{1/2} \|_2 \lor \|\BR_\natural(\BR_\natural^\top\BR_\natural)^{-1}\BSigma_\star^{1/2} \|_2 
    \leq \frac{1}{1-\varepsilon_0} .
$
\end{lemma}
\begin{proof}
\cite[Lemma~12]{tong2021accelerating} provides the following inequalities:
\[
\|\BL_\natural(\BL_\natural^\top\BL_\natural)^{-1}\BSigma_\star^{1/2} \|_2 \leq \frac{1}{1-\|\BDelta_L\BSigma_\star^{-1/2}\|_2}, \cr
\|\BR_\natural(\BR_\natural^\top\BR_\natural)^{-1}\BSigma_\star^{1/2} \|_2 \leq \frac{1}{1-\|\BDelta_R\BSigma_\star^{-1/2}\|_2},
\]
as long as $\|\BDelta_L\BSigma_\star^{-1/2}\|_2 \lor \|\BDelta_R\BSigma_\star^{-1/2}\|_2<1$. 
By Lemma~\ref{lm:Delta_F_norm}, we have 
$\|\BDelta_L\BSigma_\star^{-1/2}\|_2 \lor \|\BDelta_R\BSigma_\star^{-1/2}\|_2 \leq \varepsilon_0 \tau^k \leq\varepsilon_0$, given $\tau=1-0.6\eta<1$. The proof is finished as $\varepsilon_0=0.02<1$.
\end{proof}

\begin{lemma} \label{lm:sigma_half_Q-Q_sigma_half}  
If 
$
 \|(\BL_{k+1}\BQ_k-\BL_\star)\BSigma_\star^{1/2}\|_2 \lor \|(\BR_{k+1}\BQ_k^{-\top}-\BR_\star)\BSigma_\star^{1/2}\|_2 
\leq\varepsilon_0 \tau^{k+1} \sigma_r(\BX_\star) ,
$
then it holds 
$
 \|\BSigma_\star^{1/2}\BQ_k^{-1}(\BQ_{k+1}-\BQ_k)\BSigma_\star^{1/2}\|_2  \lor \|\BSigma_\star^{1/2}\BQ_k^\top(\BQ_{k+1}-\BQ_k)^{-\top}\BSigma_\star^{1/2}\|_2 \leq \frac{2\varepsilon_0}{1-\varepsilon_0}\sigma_r(\BX_\star).
$
\end{lemma}
\begin{proof}
\cite[Lemma~14]{tong2021accelerating} provides the inequalities:
\[
    \|\BSigma_\star^{1/2} \widetilde{\BQ}^{-1} \widehat{\BQ}\BSigma_\star^{1/2}-\BSigma_\star\|_2&\leq \frac{\|\BR(\widetilde{\BQ}^{-\top}-\widehat{\BQ}^{-\top}) \BSigma_\star^{1/2}\|_2}{1-\|(\BR\widehat{\BQ}^{-\top}-\BR_\star)\BSigma_\star^{-1/2} \|_2}, \cr
    \|\BSigma_\star^{1/2} \widetilde{\BQ}^\top \widehat{\BQ}^{-\top}\BSigma_\star^{1/2}-\BSigma_\star\|_2&\leq \frac{\|\BL(\widetilde{\BQ}-\widehat{\BQ}) \BSigma_\star^{1/2}\|_2}{1-\|(\BL\widehat{\BQ}-\BL_\star)\BSigma_\star^{-1/2} \|_2}
\]
for any $\BL,\BR\in\R^{n\times r}$ and invertible $ \widetilde{\BQ},\widehat{\BQ}\in \R^{r\times r}$, as long as $\|(\BL\widehat{\BQ}-\BL_\star)\BSigma_\star^{-1/2}\|_2 \lor \|(\BR\widehat{\BQ}^{-\top}-\BL_\star)\BSigma_\star^{-1/2}\|_2 <1$.

We will focus on the first term for now. By the assumption of this lemma and the definition of $\BQ_{k+1}$, we have
\[
    \|(\BR_{k+1}\BQ_k^{-\top}-\BR_\star)\BSigma_\star^{1/2}\|_2 &\leq \varepsilon_0 \tau^{k+1} \sigma_r(\BX_\star) ,\cr
    \|(\BR_{k+1}\BQ_{k+1}^{-\top}-\BR_\star)\BSigma_\star^{1/2}\|_2 &\leq \varepsilon_0 \tau^{k+1} \sigma_r(\BX_\star) ,\cr
    \|(\BR_{k+1}\BQ_{k+1}^{-\top}-\BR_\star)\BSigma_\star^{-1/2}\|_2 &\leq \varepsilon_0 \tau^{k+1} .
\]
Thus, by taking $\BR=\BR_{k+1}$, $\widetilde{\BQ}=\BQ_k$, and $\widehat{\BQ}=\BQ_{k+1}$, 
\[
 &\quad~\|\BSigma_\star^{1/2}\BQ_k^{-1}(\BQ_{k+1}-\BQ_k)\BSigma_\star^{1/2}\|_2\cr
 &=  \|\BSigma_\star^{1/2}\BQ_k^{-1}\BQ_{k+1}\BSigma_\star^{1/2}-\BSigma_\star\|_2  \cr
 &\leq \frac{\|\BR_{k+1}(\BQ_k^{-\top}-\BQ_{k+1}^{-\top}) \BSigma_\star^{1/2}\|_2}{1-\|(\BR_{k+1}\BQ_{k+1}^{-\top}-\BR_\star)\BSigma_\star^{-1/2} \|_2} \cr
 &\leq \frac{\|(\BR_{k+1}\BQ_k^{-\top}-\BR_\star)\BSigma_\star^{1/2}\|_2 + \|(\BR_{k+1}\BQ_{k+1}^{-\top}-\BR_\star)\BSigma_\star^{1/2}\|_2}{1 - \|(\BR_{k+1}\BQ_{k+1}^{-\top}-\BR_\star)\BSigma_\star^{-1/2}\|_2} \cr
 &\leq \frac{2\varepsilon_0 \tau^{k+1}}{1-\varepsilon_0\tau^{k+1}} \sigma_r(\BX_\star) ~
 \leq \frac{2\varepsilon_0 }{1-\varepsilon_0} \sigma_r(\BX_\star),
\]
provided $\tau=1-0.6\eta<1$. Similarly, one can see
\[
    \|\BSigma_\star^{1/2}\BQ_k^\top(\BQ_{k+1}-\BQ_k)^{-\top}\BSigma_\star^{1/2}\|_2 \leq \frac{2\varepsilon_0 }{1-\varepsilon_0} \sigma_r(\BX_\star).
\]
This finishes the proof.
\end{proof}

Notice that Lemma~\ref{lm:sigma_half_Q-Q_sigma_half} will only be used in the proof of Lemma~\ref{lm:convergence_incoher}. In the meantime, the assumption of Lemma~\ref{lm:sigma_half_Q-Q_sigma_half} is verified in \eqref{eq:dist_k+1_with_Q_k} (see the proof of Lemma~\ref{lm:convergence_dist}).

\begin{lemma} \label{lm:X-X_K_inf_norm}
If 
$
\distk \leq\varepsilon_0 \tau^k \sigma_r(\BX_\star)$ and 
$\|\BDelta_L\BSigma_\star^{1/2}\|_{2,\infty}  \lor \|\BDelta_R\BSigma_\star^{1/2}\|_{2,\infty}\leq\sqrt{{\mu r}/{n}} \tau^k \sigma_r(\BX_\star), 
$
then it holds 
$
    \|\BX_\star-\BX_k\|_\infty \leq 3\frac{\mu r}{n} \tau^k \sigma_r(\BX_\star).
$
\end{lemma}
\begin{proof}
Firstly, by Assumption~\ref{as:incoherence} and the assumptions of this lemma, we have
\[
    \|\BR_\natural\BSigma_\star^{-1/2}\|_{2,\infty} 
    &\leq \|\BDelta_R\BSigma_\star^{1/2}\|_{2,\infty} \|\BSigma_\star^{-1}\|_2 + \|\BR_\star\BSigma_\star^{-1/2} \|_{2,\infty} \cr
    &\leq \left(\tau^k+1\right)\sqrt{\frac{\mu r}{n}} ~ \leq 2\sqrt{\frac{\mu r}{n}},
\]
given $\tau=1-0.6\eta<1$. Moreover, one can see
\[
    &~\|\BX_\star-\BX_k\|_\infty \\
    =&~ \| \BDelta_L\BR_\natural^\top+\BL_\star\BDelta_R^\top \|_\infty 
    ~\leq \| \BDelta_L\BR_\natural^\top\|_\infty+\|\BL_\star\BDelta_R^\top \|_\infty \cr
    \leq&~ \| \BDelta_L\BSigma_\star^{1/2}\|_{2,\infty}\|\BR_\natural\BSigma_\star^{-1/2}\|_{2,\infty}  \cr
    &~ +\|\BL_\star\BSigma_\star^{-1/2}\|_{2,\infty}\|\BDelta_R\BSigma_\star^{1/2} \|_{2,\infty} \cr
    \leq&~ \left(2\sqrt{\frac{\mu r}{n}} + \sqrt{\frac{\mu r}{n}}\right)\sqrt{\frac{\mu r}{n}} \tau^k \sigma_r(\BX_\star)
    ~= 3 \frac{\mu r}{n} \tau^k \sigma_r(\BX_\star).
\]
This finishes the proof.
\end{proof}

\vspace{-0.1in}
\subsection{Proof of local linear convergence} \label{sec:local conv}
We will show the local convergence by first proving the claims stand at the $(k+1)$-th iteration if they stand at the $k$-th iteration.

\begin{lemma} \label{lm:convergence_dist}
If 
$
\distk \leq\varepsilon_0 \tau^k \sigma_r(\BX_\star)$ and $
\|\BDelta_L\BSigma_\star^{1/2}\|_{2,\infty}  \lor \|\BDelta_R\BSigma_\star^{1/2}\|_{2,\infty}\leq\sqrt{{\mu r}/{n}} \tau^k \sigma_r(\BX_\star),
$
then it holds 
$
\distkplusone \leq\varepsilon_0 \tau^{k+1} \sigma_r(\BX_\star).
$
\end{lemma}

\begin{proof}
Since $\BQ_{k+1}$ is the optimal alignment matrix between $(\BL_{k+1},\BR_{k+1})$ and $(\BL_\star,\BR_\star)$, so 
\[
    & ~\distsquarekplusone \\
    = &~\|(\BL_{k+1}\BQ_{k+1}-\BL_\star)\BSigma_\star^{1/2}\|_\fro^2 + \|(\BR_{k+1}\BQ_{k+1}^{-\top}-\BR_\star)\BSigma_\star^{1/2}\|_\fro^2 \cr
    \leq &~\|(\BL_{k+1}\BQ_k-\BL_\star)\BSigma_\star^{1/2}\|_\fro^2 + \|(\BR_{k+1}\BQ_k^{-\top}-\BR_\star)\BSigma_\star^{1/2}\|_\fro^2
\]
We will focus on bounding the first term in this proof, and the second term can be bounded similarly. 

Note that $\BL_\natural\BR_\natural^\top-\BX_\star=\BDelta_L\BR_\natural^\top+\BL_\star\BDelta_R^\top$. 
We have
\begin{align}\label{eq:LQ-L}
&~\quad\BL_{k+1}\BQ_k-\BL_\star \\
&= \BL_\natural-\eta(\BL_\natural\BR_\natural^\top-\BX_\star+\BS_{k+1}-\BS_\star)\BR_\natural(\BR_\natural^\top\BR_\natural)^{-1}-\BL_\star \cr
    &=\BDelta_L-\eta(\BL_\natural\BR_\natural^\top-\BX_\star)\BR_\natural(\BR_\natural^\top \BR_\natural)^{-1}-\eta\BDelta_S\BR_\natural(\BR_\natural^\top\BR_\natural)^{-1} \cr
    &=(1-\eta)\BDelta_L-\eta\BL_\star\BDelta_R^\top\BR_\natural(\BR_\natural^\top\BR_\natural)^{-1}-\eta\BDelta_S\BR_\natural(\BR_\natural^\top\BR_\natural)^{-1}.
\end{align} 
Thus,
\[
    &~\|(\BL_{k+1}\BQ_k-\BL_\star)\BSigma_\star^{1/2}\|_\fro^2 \cr =&~ \| (1-\eta)\BDelta_L\BSigma_\star^{1/2}-\eta\BL_\star\BDelta_R^\top\BR_\natural(\BR_\natural^\top\BR_\natural)^{-1}\BSigma_\star^{1/2}  \|_\fro^2 \cr
    &~
    - 2\eta(1-\eta)\tr(\BDelta_S\BR_\natural(\BR_\natural^\top\BR_\natural)^{-1}\BSigma_\star \BDelta_L^\top) \cr
    &~ +2\eta^2\tr(\BDelta_S\BR_\natural(\BR_\natural^\top\BR_\natural)^{-1}\BSigma_\star(\BR_\natural^\top\BR_\natural)^{-1}\BR_\natural^\top\BDelta_R\BL_\star^\top) \cr
    &~
    +\eta^2 \|\BDelta_S\BR_\natural(\BR_\natural^\top\BR_\natural)^{-1}\BSigma_\star^{1/2} \|_\fro^2 \cr
    :=&~ \mathfrak{R}_1 - \mathfrak{R}_2 + \mathfrak{R}_3 + \mathfrak{R}_4.
\]

\vspace{0.05in}
\noindent\textbf{Bound of $\mathfrak{R}_1$.} The component $\mathfrak{R}_1$ here is identical to $\mathfrak{R}_1$ in \cite[Section~D.1.1]{tong2021accelerating}, and the bound of this term was shown therein. We will clear this bound further by applying Lemma~\ref{lm:Delta_F_norm}, 
\[
    \mathfrak{R}_1&\leq\left((1-\eta)^2+\frac{2\varepsilon_0}{1-\varepsilon_0}\eta(1-\eta)  \right) \|\BDelta_L\BSigma_\star^{1/2}\|_\fro^2 \cr
    &\quad~ + \frac{2\varepsilon_0+\varepsilon_0^2}{(1-\varepsilon_0)^2}\eta^2\| \BDelta_R\BSigma_\star^{1/2} \|_\fro^2 \cr
    &\leq (1-\eta)^2\|\BDelta_L\BSigma_\star^{1/2}\|_\fro^2 \cr
    &\quad~ + \left((1-\eta) \frac{2\varepsilon_0^3}{1-\varepsilon_0}  +\eta\frac{2\varepsilon_0^3+\varepsilon_0^4}{(1-\varepsilon_0)^2}\right)\eta\tau^{2k}\sigma_r^2(\BX_\star).
\]

\vspace{0.05in}
\noindent\textbf{Bound of $\mathfrak{R}_2$.} Lemma~\ref{lm:sparity} implies $\BDelta_S=\BS_{k+1}-\BS_\star$ is an $\alpha$-sparse matrix. By Lemmas~\ref{lm:bound of sparse matrix}, \ref{lm:Delta_F_norm}, \ref{lm:L_R_scale_sigma_half}, \ref{lm:sparity}, and \ref{lm:X-X_K_inf_norm}, we have
\[
&\quad~|\tr(\BDelta_S\BR_\natural(\BR_\natural^\top\BR_\natural)^{-1}\BSigma_\star \BDelta_L^\top)| \cr
&\leq \|\BDelta_S\|_2 \|\BR_\natural(\BR_\natural^\top\BR_\natural)^{-1}\BSigma_\star \BDelta_L^\top\|_* \cr
&\leq \alpha n \sqrt{r} \|\BDelta_S\|_\infty  \|\BR_\natural(\BR_\natural^\top\BR_\natural)^{-1}\BSigma_\star \BDelta_L^\top\|_\fro \cr
&\leq 2\alpha n \sqrt{r} \|\BX_k-\BX_\star\|_\infty  \|\BR_\natural(\BR_\natural^\top\BR_\natural)^{-1}\BSigma_\star^{1/2}\|_2 \|\BDelta_L\BSigma_\star^{1/2}\|_\fro \cr
&\leq 6\alpha  \mu r^{1.5} \tau^{2k} \frac{\varepsilon_0}{1-\varepsilon_0}  \sigma_r^2(\BX_\star) .
\]
Hence,
$
    |\mathfrak{R}_2|\leq 12\eta(1-\eta)\alpha  \mu r^{1.5} \tau^{2k} \frac{\varepsilon_0}{1-\varepsilon_0}  \sigma_r^2(\BX_\star).
$

\vspace{0.05in}
\noindent\textbf{Bound of $\mathfrak{R}_3$.} Similar to $\mathfrak{R}_2$, we have
\[
    &~ |\tr(\BDelta_S\BR_\natural(\BR_\natural^\top\BR_\natural)^{-1}\BSigma_\star(\BR_\natural^\top\BR_\natural)^{-1}\BR_\natural^\top\BDelta_R\BL_\star^\top)| \cr
    \leq&~ \|\BDelta_S\|_2 \|\BR_\natural(\BR_\natural^\top\BR_\natural)^{-1}\BSigma_\star(\BR_\natural^\top\BR_\natural)^{-1}\BR_\natural^\top\BDelta_R\BL_\star^\top\|_* \cr
    \leq&~ \alpha n \sqrt{r} \|\BDelta_S\|_\infty \|\BR_\natural(\BR_\natural^\top\BR_\natural)^{-1}\BSigma_\star(\BR_\natural^\top\BR_\natural)^{-1}\BR_\natural^\top\BDelta_R\BL_\star^\top\|_\fro \cr
    \leq&~ \alpha n \sqrt{r} \|\BDelta_S\|_\infty \|\BR_\natural(\BR_\natural^\top\BR_\natural)^{-1}\BSigma_\star^{1/2}\|_2^2\|\BDelta_R\BL_\star^\top\|_\fro \cr
    \leq&~ 2\alpha n \sqrt{r} \|\BX_k-\BX_\star\|_\infty \|\BR_\natural(\BR_\natural^\top\BR_\natural)^{-1}\BSigma_\star^{1/2}\|_2^2\|\BDelta_R\BSigma_\star^{1/2}\|_\fro \cr
    \leq&~ 6\alpha  \mu r^{1.5} \tau^{2k} \frac{\varepsilon_0}{(1-\varepsilon_0)^2}  \sigma_r^2(\BX_\star) .
\]
Hence, 
$
    |\mathfrak{R}_3|\leq 12\eta^2\alpha  \mu r^{1.5} \tau^{2k} \frac{\varepsilon_0}{(1-\varepsilon_0)^2}  \sigma_r^2(\BX_\star).
$

\vspace{0.05in}
\noindent\textbf{Bound of $\mathfrak{R}_4$.}
\[
    &~ \|\BDelta_S\BR_\natural(\BR_\natural^\top\BR_\natural)^{-1}\BSigma_\star^{1/2} \|_\fro^2  \leq r  \|\BDelta_S\BR_\natural(\BR_\natural^\top\BR_\natural)^{-1}\BSigma_\star^{1/2} \|_2^2 \cr
    &~\leq r\|\BDelta_S\|_2^2 \|\BR_\natural(\BR_\natural^\top\BR_\natural)^{-1}\BSigma_\star^{1/2} \|_2^2 \cr
    &~\leq 4\alpha^2 n^2 r \|\BX_k-\BX_\star\|_\infty^2 \|\BR_\natural(\BR_\natural^\top\BR_\natural)^{-1}\BSigma_\star^{1/2} \|_2^2 \cr
    &~\leq 36\alpha^2 \mu^2 r^3 \tau^{2k} \frac{1}{(1-\varepsilon_0)^2}  \sigma_r^2(\BX_\star).
\]
Hence,
$
    \mathfrak{R}_4\leq 36\eta^2\alpha^2 \mu^2 r^3 \tau^{2k} \frac{1}{(1-\varepsilon_0)^2}  \sigma_r^2(\BX_\star).
$

Combine all the bounds together, we have
\[
    \|(\BL_{k+1}\BQ_k&-\BL_\star)\BSigma_\star^{1/2}\|_\fro^2 
    \leq  (1-\eta)^2\|\BDelta_L\BSigma_\star^{1/2}\|_\fro^2 \cr 
    &~ + \left((1-\eta) \frac{2\varepsilon_0^3}{1-\varepsilon_0}  +\eta\frac{2\varepsilon_0^3+\varepsilon_0^4}{(1-\varepsilon_0)^2}\right)\eta\tau^{2k}\sigma_r^2(\BX_\star) \cr
    &~ + 12\eta(1-\eta)\alpha  \mu r^{1.5} \tau^{2k} \frac{\varepsilon_0}{1-\varepsilon_0}  \sigma_r^2(\BX_\star) \cr
    &~ + 12\eta^2\alpha  \mu r^{1.5} \tau^{2k} \frac{\varepsilon_0}{(1-\varepsilon_0)^2}  \sigma_r^2(\BX_\star) \cr
    &~ + 36\alpha^2 \mu^2 r^3 \tau^{2k} \frac{1}{(1-\varepsilon_0)^2}  \sigma_r^2(\BX_\star),
\]
and a similar bound can be computed for $\|(\BR_{k+1}\BQ_k^{-\top}-\BR_\star)\BSigma_\star^{1/2}\|_\fro^2$. Put together, we have
\[\label{eq:dist_k+1_with_Q_k}
    &~\distsquarekplusone \cr
    \leq&~ \|(\BL_{k+1}\BQ_k-\BL_\star)\BSigma_\star^{1/2}\|_\fro^2 + \|(\BR_{k+1}\BQ_k^{-\top}-\BR_\star)\BSigma_\star^{1/2}\|_\fro^2 \cr
    \leq&~  (1-\eta)^2\left(\|\BDelta_L\BSigma_\star^{1/2}\|_\fro^2+\|\BDelta_R\BSigma_\star^{1/2}\|_\fro^2\right) \cr
    &~ + 2\left((1-\eta) \frac{2\varepsilon_0^3}{1-\varepsilon_0}  +\eta\frac{2\varepsilon_0^3+\varepsilon_0^4}{(1-\varepsilon_0)^2}\right)\eta\tau^{2k}\sigma_r^2(\BX_\star) \cr
    &~ + 24\eta(1-\eta)\alpha  \mu r^{1.5} \tau^{2k} \frac{\varepsilon_0}{1-\varepsilon_0}  \sigma_r^2(\BX_\star) \cr
    &~ + 24\eta^2\alpha  \mu r^{1.5} \tau^{2k} \frac{\varepsilon_0}{(1-\varepsilon_0)^2}  \sigma_r^2(\BX_\star) \cr
    &~ + 72\alpha^2 \mu^2 r^3 \tau^{2k} \frac{1}{(1-\varepsilon_0)^2}  \sigma_r^2(\BX_\star)\cr
    \leq&~  \Bigg( (1-\eta)^2 + 2\left((1-\eta) \frac{2\varepsilon_0}{1-\varepsilon_0}  +\eta\frac{2\varepsilon_0+\varepsilon_0^2}{(1-\varepsilon_0)^2}\right)\eta \cr
    &~ + 24\eta(1-\eta)\alpha  \mu r^{1.5}  \frac{1}{\varepsilon_0(1-\varepsilon_0)} + 24\eta^2\alpha  \mu r^{1.5}  \frac{1}{\varepsilon_0(1-\varepsilon_0)^2}  \cr
    &~ + 72\alpha^2 \mu^2 r^3  \frac{1}{\varepsilon_0^2(1-\varepsilon_0)^2}  \Bigg)\varepsilon_0^2\tau^{2k} \sigma_r^2(\BX_\star) \cr
    \leq&~ (1-0.6\eta)^2 \varepsilon_0^2\tau^{2k} \sigma_r^2(\BX_\star) ,
\]
where we use the fact $\|\BDelta_L\BSigma_\star^{1/2}\|_\fro^2+\|\BDelta_R\BSigma_\star^{1/2}\|_\fro^2 =: \distsquarek \leq \varepsilon_0^2\tau^{2k}\sigma_r^2(\BX_\star)$ in the second step, and the last step use $\varepsilon_0=0.02$, $\alpha\leq\frac{1}{10^4\mu r^{1.5}}$, and $\frac{1}{4}\leq\eta\leq \frac{8}{9}$. The proof is finished by substituting $\tau=1-0.6\eta$. 
%
\end{proof}

\begin{lemma} \label{lm:convergence_incoher}
If 
$
\distk \leq\varepsilon_0 \tau^k \sigma_r(\BX_\star)$ and $
\|\BDelta_L\BSigma_\star^{1/2}\|_{2,\infty}  \lor \|\BDelta_R\BSigma_\star^{1/2}\|_{2,\infty}\leq\sqrt{{\mu r}/{n}} \tau^k \sigma_r(\BX_\star),
$
then it holds
$
\|(\BL_{k+1}\BQ_{k+1}-\BL_\star)\BSigma_\star^{1/2}\|_{2,\infty}
\lor \|(\BR_{k+1}\BQ_{k+1}^{-\top}-\BR_\star)\BSigma_\star^{1/2}\|_{2,\infty} \leq  \sqrt{{\mu r}/{n}} \tau^{k+1} \sigma_r(\BX_\star).
$
\end{lemma}
\begin{proof}
Using \eqref{eq:LQ-L} again, we have
\[
    &~\|(\BL_{k+1}\BQ_k-\BL_\star)\BSigma_\star^{1/2}\|_{2,\infty} \cr
    \leq &~ (1-\eta)\|\BDelta_L\BSigma_\star^{1/2}\|_{2,\infty}
    +\eta\|\BL_\star\BDelta_R^\top\BR_\natural(\BR_\natural^\top\BR_\natural)^{-1}\BSigma_\star^{1/2}\|_{2,\infty} \cr
    &~ +\eta\|\BDelta_S\BR_\natural(\BR_\natural^\top\BR_\natural)^{-1}\BSigma_\star^{1/2}\|_{2,\infty} \cr
    :=&~ \mathfrak{T}_1 + \mathfrak{T}_2 +\mathfrak{T}_3.
\]

\vspace{0.05in}
\noindent\textbf{Bound of $\mathfrak{T}_1$.} $\mathfrak{T}_1\leq(1-\eta)\sqrt{{\mu r}/{n}}\tau^k\sigma_r(\BX_\star)$ is directly followed by the assumption of this lemma.

\vspace{0.05in}
\noindent\textbf{Bound of $\mathfrak{T}_2$.} Assumption~\ref{as:incoherence} implies $\|\BL_\star\BSigma_\star^{-1/2}\|_{2,\infty}\leq\sqrt{{\mu r}/{n}}$ and Lemma~\ref{lm:Delta_F_norm} implies $\|\BDelta_R\BSigma_\star^{1/2}\|_2\leq \tau^k \varepsilon_0$. 
Together with  Lemma~\ref{lm:L_R_scale_sigma_half} 
, we have
\[
    \mathfrak{T}_2
    &\leq \eta\|\BL_\star\BSigma_\star^{-1/2}\|_{2,\infty} \|\BDelta_R\BSigma_\star^{1/2}\|_2\|\BR_\natural(\BR_\natural^\top\BR_\natural)^{-1}\BSigma_\star^{1/2}\|_2 \cr
    &\leq \eta \frac{\varepsilon_0}{1-\varepsilon_0}\sqrt{\frac{\mu r}{n}}\tau^k\sigma_r(\BX_\star) .
\]

\vspace{0.05in}
\noindent\textbf{Bound of $\mathfrak{T}_3$.} By Lemma~\ref{lm:sparity}, $\supp(\BDelta_S)\subseteq\supp(\BS_\star)$, which implies that $\BDelta_S$ is an $\alpha$-sparse matrix. 
By Lemma~\ref{lm:bound of sparse matrix}, we have
\[
    \mathfrak{T}_3 
    &\leq \eta\|\BDelta_S\|_{2,\infty} \|\BR_\natural(\BR_\natural^\top\BR_\natural)^{-1}\BSigma_\star^{1/2}\|_2 \cr
    &\leq \eta\frac{\sqrt{\alpha n}}{1-\varepsilon_0}\|\BDelta_S\|_\infty 
    ~\leq 2\eta\frac{\sqrt{\alpha n}}{1-\varepsilon_0}\|\BX_\star-\BX_k\|_\infty \cr
    &\leq 6\eta\frac{\sqrt{\alpha \mu r}}{1-\varepsilon_0} \sqrt{\frac{\mu r}{n}} \tau^k \sigma_r(\BX_\star),
\]
where the last two steps use Lemmas~\ref{lm:sparity} and \ref{lm:X-X_K_inf_norm}. Together, 
\[ 
    &~\|(\BL_{k+1}\BQ_k-\BL_\star)\BSigma_\star^{1/2}\|_{2,\infty}
    ~\leq \mathfrak{T}_1 + \mathfrak{T}_2 +\mathfrak{T}_3 \cr
    \leq& \left(  
    1-\eta
    + \eta \frac{\varepsilon_0}{1-\varepsilon_0}
    + 6\eta\frac{\sqrt{\alpha \mu r}}{1-\varepsilon_0} \right)
    \sqrt{\frac{\mu r}{n}} \tau^k \sigma_r(\BX_\star) \cr
    \leq& \left(  
    1-\eta\left(1
    - \frac{\varepsilon_0}{1-\varepsilon_0}
    - 6\frac{\sqrt{\alpha \mu r}}{1-\varepsilon_0}\right) \right)
    \sqrt{\frac{\mu r}{n}} \tau^k \sigma_r(\BX_\star). \label{eq:LQ-L_sigma_half_2_inf}
\]
In addition, we also have 
\[\label{eq:LQ-L_sigma_-half_2_inf}
&~
     \|(\BL_{k+1}\BQ_k-\BL_\star)\BSigma_\star^{-1/2}\|_{2,\infty} \cr
     \leq& \left(  
    1-\eta\left(1
    - \frac{\varepsilon_0}{1-\varepsilon_0}
    - 6\frac{\sqrt{\alpha \mu r}}{1-\varepsilon_0}\right) \right)
    \sqrt{\frac{\mu r}{n}} \tau^k.
\]

\vspace{0.05in}
\noindent\textbf{Bound with $\BQ_{k+1}$.} Note that $\BQ$'s are the best align matrices under Frobenius norm but this is not necessarily true under $\ell_{2,\infty}$ norm. So we must show the bound of $\|(\BL_{k+1}\BQ_{k+1}-\BL_\star)\BSigma_\star^{1/2}\|_{2,\infty}$ directly. 
Note that $\BQ_{k+1}$ does exist, according to \cite[Lemma~9]{tong2021accelerating}.
Applying \eqref{eq:LQ-L_sigma_half_2_inf}, \eqref{eq:LQ-L_sigma_-half_2_inf} and Lemma~\ref{lm:sigma_half_Q-Q_sigma_half} gives
\[
   &~ \|(\BL_{k+1}\BQ_{k+1}-\BL_\star)\BSigma_\star^{1/2}\|_{2,\infty} \cr 
   \leq&~ \|(\BL_{k+1}\BQ_k-\BL_\star)\BSigma_\star^{1/2}\|_{2,\infty} \cr 
   &~+ \|\BL_{k+1}(\BQ_{k+1}-\BQ_k)\BSigma_\star^{1/2}\|_{2,\infty} \cr
   =&~ \|(\BL_{k+1}\BQ_k-\BL_\star)\BSigma_\star^{1/2}\|_{2,\infty} \cr
   &~ + \|\BL_{k+1}\BQ_k\BSigma_\star^{-1/2}\BSigma_\star^{1/2}\BQ_k^{-1}(\BQ_{k+1}-\BQ_k)\BSigma_\star^{1/2}\|_{2,\infty}  \cr
   \leq&~ \|(\BL_{k+1}\BQ_k-\BL_\star)\BSigma_\star^{1/2}\|_{2,\infty} \cr 
   &~+ \|\BL_{k+1}\BQ_k\BSigma_\star^{-1/2}\|_{2,\infty} \|\BSigma_\star^{1/2}\BQ_k^{-1}(\BQ_{k+1}-\BQ_k)\BSigma_\star^{1/2}\|_2 \cr
   \leq&~ \|(\BL_{k+1}\BQ_k-\BL_\star)\BSigma_\star^{1/2}\|_{2,\infty} \cr
   &~ +(\|(\BL_{k+1}\BQ_k-\BL_\star)\BSigma_\star^{-1/2}\|_{2,\infty}  \cr
   &~ +\|\BL_\star\BSigma_\star^{-1/2} \|_{2,\infty}) \|\BSigma_\star^{1/2}\BQ_k^{-1}(\BQ_{k+1}-\BQ_k)\BSigma_\star^{1/2}\|_2 
   \cr
   \leq&~ \sqrt{\frac{\mu r}{n}} \tau^k \sigma_r(\BX_\star)\Bigg(  1-\eta\left(1
    - \frac{\varepsilon_0}{1-\varepsilon_0}
    - 6\frac{\sqrt{\alpha \mu r}}{1-\varepsilon_0}\right) \cr 
    &~ + \frac{2\varepsilon_0}{1-\varepsilon_0} \left(  2-\eta\left(1
    - \frac{\varepsilon_0}{1-\varepsilon_0}
    - 6\frac{\sqrt{\alpha \mu r}}{1-\varepsilon_0}\right) \right)\Bigg) \cr
    \leq&~ (1-0.6\eta)\sqrt{\frac{\mu r}{n}} \tau^k \sigma_r(\BX_\star) ,
\]
where the last step use $\varepsilon_0=0.02$, $\alpha\leq\frac{1}{10^4\mu r^{1.5}}$, and $\frac{1}{4}\leq\eta\leq \frac{8}{9}$. Similar result for $\|(\BR_{k+1}\BQ_{k+1}^{-\top}-\BR_\star)\BSigma_\star^{1/2}\|_{2,\infty}$. 
The proof is finished by substituting $\tau=1-0.6\eta$. 
\end{proof}

Now we have all the ingredients for proving the theorem of local linear convergence, i.e., Theorem~\ref{thm:local convergence}.
\begin{proof}[Proof of Theorem~\ref{thm:local convergence}]
This proof is done by induction.
\noindent\textbf{Base case.} Since $\tau^0=1$, the assumed initial conditions satisfy the base case at $k=0$.
\noindent\textbf{Induction step.} At the $k$-th iteration, we assume the conditions
\[
&~ 
\distk \leq \varepsilon_0 \tau^k \sigma_r(\BX_\star), \cr
&~ \|(\BL_k\BQ_k-\BL_\star)\BSigma_\star^{1/2}\|_{2,\infty}  \lor \|(\BR_k\BQ_k^{-\top}-\BR_\star)\BSigma_\star^{1/2}\|_{2,\infty} \cr 
\leq &~ \sqrt{{\mu r}/{n}} \tau^k \sigma_r(\BX_\star)
\]
hold, then by Lemmas~\ref{lm:convergence_dist} and \ref{lm:convergence_incoher},
\[
 &\qquad\distkplusone \leq\varepsilon_0 \tau^{k+1} \sigma_r(\BX_\star), \cr
&\|(\BL_{k+1}\BQ_{k+1}-\BL_\star)\BSigma_\star^{1/2}\|_{2,\infty} \cr 
&\lor \|(\BR_{k+1}\BQ_{k+1}^{-\top}-\BR_\star)\BSigma_\star^{1/2}\|_{2,\infty}  
\leq  \sqrt{{\mu r}/{n}} \tau^{k+1} \sigma_r(\BX_\star)
\]
also hold. This finishes the proof.
\end{proof}

\vspace{-0.1in}
\subsection{Proof of guaranteed initialization} \label{sec:guaranteed initialization}
Now we show that the outputs of the initialization step in Algorithm~\ref{algo:LRMC} satisfy the initial conditions required by Theorem~\ref{thm:local convergence}.
\begin{proof}[Proof of Theorem~\ref{thm:initial}]
Firstly, by Assumption~\ref{as:incoherence}, we obtain
\[
    \|\BX_\star\|_\infty\leq\|\BU_\star\|_{2,\infty}\|\BSigma_\star\|_2\|\BV_\star\|_{2,\infty}\leq\frac{\mu r}{n}\sigma_1(\BX_\star).
\]
Invoking Lemma~\ref{lm:sparity} with $\BX_{-1}=\bm{0}$, we have 
\[ 
    \|\BS_\star-\BS_0\|_\infty\leq 2\frac{\mu r}{n}\sigma_1(\BX_\star) \quad\textnormal{and}\quad \supp(\BS_0)\subseteq\supp(\BS_\star),
\]
which implies $\BS_\star-\BS_0$ is $\alpha$-sparse. Applying Lemma~\ref{lm:bound of sparse matrix} to get 
\[
    \|\BS_\star-\BS_0\|_2\leq \alpha n\|\BS_\star-\BS_0\|_\infty
    \leq2 \alpha \mu r \kappa \sigma_r(\BX_\star).
\]
Since $\BX_0=\BL_0\BR_0^\top$ is the best rank-$r$ approximation of $\BY-\BS_0$, 
\[
    \|\BX_\star-\BX_0\|_2
    &\leq \|\BX_\star-(\BY-\BS_0)\|_2 + \|(\BY-\BS_0) -\BX_0 \|_2 \cr
    &\leq 2\|\BX_\star-(\BY-\BS_0)\|_2 \cr
    &= 2\|\BS_\star-\BS_0\|_2 ~
    \leq 4 \alpha \mu r \kappa \sigma_r(\BX_\star),
\]
where the equality uses the definition $\BY=\BX_\star+\BS_\star$. By \cite[Lemma~11]{tong2021accelerating}, we obtain
\[
    \distzero &\leq \sqrt{\sqrt{2}+1}\|\BX_\star-\BX_0\|_\fro \cr
    &\leq \sqrt{(\sqrt{2}+1)2r}\|\BX_\star-\BX_0\|_2 \cr
    &\leq 10  \alpha \mu r^{1.5} \kappa \sigma_r(\BX_\star),
\]
where we use the fact that $\BX_\star-\BX_0$ has at most rank-$2r$. Given $\alpha\leq\frac{c_0}{\mu r^{1.5}\kappa}$, we prove the first claim:
\[ \label{eq:init_dist}
    \distzero \leq 10 c_0 \sigma_r(\BX_\star).
\]

Let $\varepsilon_0:=10c_0$. Now, we will prove the second claim:
\[
\|\BDelta_L\BSigma_\star^{1/2}\|_{2,\infty} \lor \|\BDelta_R\BSigma_\star^{1/2}\|_{2,\infty} \leq \sqrt{{\mu r}/{n}} \sigma_r(\BX_\star),
\]
where $\BDelta_L:=\BL_0\BQ_0-\BL_\star$ and $\BDelta_R:=\BR_0\BQ_0^{-\top}-\BR_\star$. For ease of notation, we also denote $\BL_\natural=\BL_0\BQ_0$, $\BR_\natural=\BR_0\BQ_0^{-\top}$, and $\BDelta_S=\BS_0-\BS_\star$ in the rest of this proof. 

We will work on $\|\BDelta\BSigma_\star^{1/2}\|_{2,\infty}$ first, and $\|\BDelta\BSigma_\star^{1/2}\|_{2,\infty}$ can be bounded similarly. 

Since $\BU_0\BSigma_0\BV_0^{\top}=\SVD_r(\BY-\BS_0)=\SVD_r(\BX_\star-\BDelta_S)$,  
\[
    \BL_0 = \BU_0\BSigma_0^{1/2} 
    &= (\BX_\star-\BDelta_S)\BV_0\BSigma_0^{-1/2} 
    = (\BX_\star-\BDelta_S)\BR_0\BSigma_0^{-1} \cr
    &= (\BX_\star-\BDelta_S)\BR_0(\BR_0^\top\BR_0)^{-1}.
\]
Multiplying $\BQ_0\BSigma_\star^{1/2}$ on both sides, we have
\[
    \BL_\natural\BSigma_\star^{1/2} = \BL_0\BQ_0\BSigma_\star^{1/2} 
    &= (\BX_\star-\BDelta_S)\BR_0(\BR_0^\top\BR_0)^{-1} \BQ_0\BSigma_\star^{1/2} \cr
    &= (\BX_\star-\BDelta_S)\BR_\natural(\BR_\natural^\top\BR_\natural)^{-1} \BSigma_\star^{1/2}.
\]
Subtracting $\BX_\star\BR_\natural(\BR_\natural^\top\BR_\natural)^{-1}\BSigma_\star^{1/2}$ on both sides, we have
\[
    &~ \BL_\natural\BSigma_\star^{1/2} - \BL_\star\BR_\star^\top\BR_\natural(\BR_\natural^\top\BR_\natural)^{-1}\BSigma_\star^{1/2} \cr 
    = &~ (\BX_\star-\BDelta_S)\BR_\natural(\BR_\natural^\top\BR_\natural)^{-1} \BSigma_\star^{1/2} -\BX_\star\BR_\natural(\BR_\natural^\top\BR_\natural)^{-1}\BSigma_\star^{1/2} \cr
    &~ \BDelta_L\BSigma_\star^{1/2}+\BL_\star\BDelta_R^\top  \BR_\natural(\BR_\natural^\top\BR_\natural)^{-1} \BSigma_\star^{1/2} \cr
    = &~ -\BDelta_S\BR_\natural(\BR_\natural^\top\BR_\natural)^{-1} \BSigma_\star^{1/2} ,
\]
where the left operand of last step uses the fact  $\BL_\star\BSigma_\star^{1/2} =\BL_\star\BR_\natural^\top\BR_\natural(\BR_\natural^\top\BR_\natural)^{-1}\BSigma^{1/2}$.
Thus, 
\[
\| \BDelta_L\BSigma_\star^{1/2} \|_{2,\infty} 
\leq&~ \|\BL_\star\BDelta_R^\top  \BR_\natural(\BR_\natural^\top\BR_\natural)^{-1} \BSigma_\star^{1/2}\|_{2,\infty} \cr 
 &~ + \|\BDelta_S\BR_\natural(\BR_\natural^\top\BR_\natural)^{-1} \BSigma_\star^{1/2}\|_{2,\infty} \cr
    :=&~ \mathfrak{J}_1 + \mathfrak{J}_2.
\]

\vspace{0.05in}
\noindent\textbf{Bound of $\mathfrak{J}_1$.} By Assumption~\ref{as:incoherence}, we get
\[
    \mathfrak{J}_1 &\leq \|\BL_\star\BSigma_\star^{-1/2}\|_{2,\infty} \|\BDelta_R\BSigma_\star^{1/2}\|_2  \|\BR_\natural(\BR_\natural^\top\BR_\natural)^{-1} \BSigma_\star^{1/2}\|_2 \cr
    &\leq \sqrt{\frac{\mu r}{n}}\frac{\varepsilon_0}{1-\varepsilon_0}\sigma_r(\BX_\star),
\]
where Lemma~\ref{lm:Delta_F_norm} implies $\|\BDelta_R\BSigma_\star^{1/2}\|_2 \leq \varepsilon_0 \sigma_r(\BX_\star)$, and Lemma~\ref{lm:L_R_scale_sigma_half} implies $\|\BR_\natural(\BR_\natural^\top\BR_\natural)^{-1} \BSigma_\star^{1/2}\|_2 \leq\frac{1}{1-\varepsilon_0}$, given \eqref{eq:init_dist}.

\vspace{0.05in}
\noindent\textbf{Bound of $\mathfrak{J}_2$.} 
Notice that $\BDelta_S$ is $\alpha$-sparse. Moreover, by \eqref{eq:init_dist}, Lemmas~\ref{lm:bound of sparse matrix} and \ref{lm:L_R_scale_sigma_half}, we have
\[
    \mathfrak{J}_2 &\leq \|\BDelta_S\|_{1,\infty}\|\BR_\natural\BSigma_\star^{-1/2}\|_{2,\infty}\|\BSigma_\star^{1/2}(\BR_\natural^\top\BR_\natural)^{-1} \BSigma_\star^{1/2}\|_2 \cr
    &\leq \alpha n\|\BDelta_S\|_\infty\|\BR_\natural\BSigma_\star^{-1/2}\|_{2,\infty}\|\BR_\natural(\BR_\natural^\top\BR_\natural)^{-1} \BSigma_\star^{1/2}\|_2^2 \cr
    &\leq \alpha n \frac{2 \mu r}{n}\sigma_1(\BX_\star)\frac{1}{(1-\varepsilon_0)^2}\|\BR_\natural\BSigma_\star^{-1/2}\|_{2,\infty} \cr
    &\leq \frac{2 \alpha \mu r \kappa}{(1-\varepsilon_0)^2}\left(\sqrt{\frac{\mu r}{n}}+\|\BDelta_R\BSigma_\star^{-1/2}\|_{2,\infty}\right) \sigma_r(\BX_\star),
\]
where the first step uses  $\|\bm{A}\bm{B}\|_{2,\infty}\leq\|\bm{A}\|_{1,\infty}\|\bm{B}\|_{2,\infty}$. 
Note that $\|\BDelta_R\BSigma_\star^{-1/2}\|_{2,\infty}\leq \frac{\|\BDelta_R\BSigma_\star^{1/2}\|_{2,\infty}}{\sigma_r(\BX_\star)}$. 
Hence, 
\[
    \| \BDelta_L\BSigma_\star^{1/2} \|_{2,\infty} 
    \leq &~ \left(\frac{\varepsilon_0}{1-\varepsilon_0}+\frac{2 \alpha \mu r \kappa}{(1-\varepsilon_0)^2}\right) \sqrt{\frac{\mu r}{n}} \sigma_r(\BX_\star) \cr 
    &~ + \frac{2 \alpha \mu r \kappa}{(1-\varepsilon_0)^2}\|\BDelta_R\BSigma_\star^{1/2}\|_{2,\infty} ,
\]
and similarly, one can see
\[
    \| \BDelta_R\BSigma_\star^{1/2} \|_{2,\infty} 
    \leq &~ \left(\frac{\varepsilon_0}{1-\varepsilon_0}+\frac{2 \alpha \mu r \kappa}{(1-\varepsilon_0)^2}\right) \sqrt{\frac{\mu r }{n}} \sigma_r(\BX_\star) \cr 
    &~ + \frac{2 \alpha \mu r \kappa}{(1-\varepsilon_0)^2}\|\BDelta_L\BSigma_\star^{1/2}\|_{2,\infty}  .
\]
Therefore, substituting $\varepsilon_0=10 c_0$ gives
\[
    &~ \| \BDelta_L\BSigma_\star^{1/2} \|_{2,\infty} \lor \| \BDelta_R\BSigma_\star^{1/2} \|_{2,\infty} 
    ~\leq \mathfrak{J}_1 + \mathfrak{J}_2 \cr
    \leq &~ \frac{(1-\varepsilon_0)^2}{(1-\varepsilon_0)^2-2\alpha \mu r \kappa}\left(\frac{\varepsilon_0}{1-\varepsilon_0}+\frac{2 \alpha \mu r \kappa}{(1-\varepsilon_0)^2}\right) \sqrt{\frac{\mu r }{n}} \sigma_r(\BX_\star) \cr
    \leq &~ \sqrt{{\mu r}/{n}} \sigma_r(\BX_\star),
\]
as long as $c_0\leq \frac{1}{35}$. 
This finishes the proof.
\end{proof}


\section{Conclusion}
This paper introduces a novel scalable and learnable approach to large-scale Robust Matrix Completion problems, coined Learned Robust Matrix Completion (LRMC). It is paired with a novel feedforward-recurrent-mixed neural network model that flexibly learns the algorithm for potentially infinite iterations without retraining or compromising performance. 
We theoretically reveal the learning potential of the proposed LRMC. Through extensive numerical experiments, we demonstrate the superiority of LRMC over existing state-of-the-art approaches.

\begin{small}
\bibliographystyle{IEEEtran}
\bibliography{IEEEabrv,ref}
\end{small}

\end{document}